\documentclass[11pt]{article}      
\usepackage{amsmath,amssymb,amsfonts} 
\usepackage{graphics}                 
\usepackage{color}                    
\usepackage{hyperref}                 
\usepackage{amsthm}
\usepackage{natbib}\citeindextrue
\usepackage[pdftex]{graphicx}
\usepackage{comment}
\usepackage[left=1in,top=1in,bottom=1in,right=1in]{geometry}
\usepackage[algoruled,lined]{algorithm2e}
\usepackage{algorithmic}
\usepackage{multirow}
\usepackage{enumerate}

\newcommand{\R}{\mathbb{R}}
\newcommand{\p}{\mathbb{P}}
\newcommand{\E}{\mathbb{E}}
\newcommand{\1}{\mathbf{1}}
\newcommand{\T}{\mathcal{T}}

\newcommand{\G}{\mathbb{G}}

\newcommand{\g}{\, | \,}
\newcommand{\mysec}[1]{Section~\ref{sec:#1}}
\newcommand{\myfig}[1]{Figure~\ref{fig:#1}}
\newtheorem{theorem}{Theorem}[section]
\newtheorem{lemma}[theorem]{Lemma}

\newtheorem{definition}[theorem]{Definition}

\begin{document}

\title{Dirichlet Process Mixtures of Generalized Linear Models}

\author{Lauren A. Hannah, David M. Blei, Warren B. Powell}

\maketitle

\begin{abstract}
  We propose Dirichlet Process mixtures of Generalized Linear Models (DP-GLM), a new method of nonparametric regression that accommodates continuous and categorical inputs, and responses that can be modeled by a generalized linear model.  We prove conditions for the asymptotic unbiasedness of the DP-GLM regression mean function estimate.  We also give examples for when those conditions hold, including models for compactly supported continuous distributions and a model with continuous covariates and categorical response.  We empirically analyze the properties of the DP-GLM and why it provides better results than existing Dirichlet process mixture regression models.  We evaluate DP-GLM on several data sets, comparing it to modern methods of nonparametric regression like CART, Bayesian trees and Gaussian processes.  Compared to existing techniques, the DP-GLM provides a single model (and corresponding inference algorithms) that performs well in many regression settings.
  \end{abstract}

\section{Introduction}
\label{sec:intro}

In this paper, we examine the general regression problem.  The general
regression problem models a response variable $Y$ as dependent on a
set of covariates $x$,
\begin{equation}
  Y \, | \, x \sim f(m(x)).
\end{equation}
The function $m(x)$ is called the \textit{mean function}, which maps
the covariates to the conditional mean of the response; the
distribution $f$ characterizes the deviation of the response from its
conditional mean.  The simplest example of general regression is
linear regression, where $m(x)$ is a linear function of $x$ and $f$ is
a Gaussian distribution with mean $m(x)$ and fixed variance.


The linear regression methodology is generalized to many types of
response variables with \textit{generalized linear models} (GLMs)~\citep{McNe89}.  In
their canonical form, a GLM assumes that the conditional mean of the
response is a linear function of the covariates and that the response
distribution is in an exponential family.  GLMs generalize many
classical regression and classification methods beyond linear
regression, including logistic regression, multinomial regression, and
Poisson regression.


A considerable restriction imposed by a GLM is that the covariates
must enter the distribution of the response through a linear function.
(A non-linear function can be applied to the output of the linear
function, but only one that does not depend on the covariates.)  For real world applications where the distribution of the response depends on the covariates in a non-linear way, this assumption is limiting.  Flexibly fitting non-linear response functions is the
problem of \textit{nonparametric regression}.

Our goal in this paper is to develop a general-purpose method for
nonparametric regression.  We develop an algorithm that can capture
arbitarily shaped response functions, model diverse response types and
covariate types, accommodate high dimensional covariates, and capture
heteroscedasticity, i.e., the property of the response distribution
where both its mean and variance change with the covariates.

Our idea is to model $m(x)$ by a mixture of simpler ``local'' response
distributions $f_i(m_i(x))$, each one applicable in a region of the
covariates that exhibits similar response patterns.  To handle
multiple types of responses, each local regression is a GLM.  Notice
this means that each $m_i(x)$ is a linear function---the desired
non-linear mean function arises when we marginalize out the
uncertainty about which local response distribution is in play.  (See
\myfig{demo} for a simple example with one covariate and a continuous
response function.)  Furthermore, our method captures
heteroscedasticity.  Each GLM $f_i$ can vary in a way beyond the
variability that arises from a single linear function of the
covariates.

Finally, we take a Bayesian nonparametric approach to determining the
number of local regressions needed to explain, and form predictions
about, a particular data set.  With a Bayesian nonparametric mixture
model, we let the data determine both the number and form of simple
mean functions that are mixed.  This is critical for the objective
of modeling arbitrary response distributions: complex response
functions can be constructed with many local functions, while simple
response functions need only a small number.  Unlike frequentist
nonparametric regression methods, e.g., those that create a mean
function for each data point, the Bayesian nonparametric approach is
biased to using only as complex a model as the data allow.

Thus, we develop \textit{Dirichlet process mixtures of generalized
  linear models} (DP-GLMs), a Bayesian nonparametric regression model
that combines the advantages of generalized linear models with the
flexibility of nonparametric regression.  DP-GLMs are a generalization
of several existing DP-based, covariate/response specific regression models~\citep{MuErWe96,ShNe07}
to a variety of response distributions.  We derive Gibbs sampling
algorithms for fitting and predicting with DP-GLMs.  We investigate
some of the statistical properties of these models, such as the form
of their posterior and conditions for the asymptotic unbiasedness of
their predictions.  We study DP-GLMs with several types of data.

In addition to defining and discussing the DP-GLM, a central
contribution of this paper is our theoretical analysis of its response
estimator and, specifically, the asymptotic unbiasedness of its
predictions.  Asymptotic properties help justify the use of certain
regression models, but they have largely been ignored for regression
models with Dirichlet process priors.  We will give general conditions
for asymptotic unbiasedness, and examples of when they are satisfied.
(These conditions are model-dependent, and can be difficult to check.)

The rest of this paper is organized as follows.  In Section \ref{sec:related-work}, we review the current research literature on Bayesian nonparametric regression and highlight how the DP-GLM extends this
field.  In Section \ref{sec:mathReview}, we review
Dirichlet process mixture models and generalized linear models. In
Section \ref{sec:dpglm}, we construct the DP-GLM and derive algorithms
for posterior computation.  In Section \ref{sec:theory} we
give general conditions for unbiasedness and prove it in a specific
case with conjugate priors. In Section \ref{sec:numbers} we compare
DP-GLM and existing methods on three data sets.  We illustrate that
the DP-GLM provides a powerful nonparametric regression model that can
accommodate many data analysis settings.

\section{Related work}
\label{sec:related-work}

Gaussian process (GP), Bayesian regression trees and Dirichlet process
mixtures are the most common prior choices for Bayesian nonparametric
regression.  GP priors assume that the observations arise from a
Gaussian process model with known covariance function form (see
\citet{RaWi06} for a review).  Without modification, however, the GP
model is only applicable to problems with continuous covariates and
constant variance.  The assumption of constant covariance can be eased
by using Dirichlet process mixtures of GPs~\citep{RaGh02} or treed
GPs~\citep{GrLe08}.  Bayesian regression trees place a prior over the
size of the tree and can be viewed as an automatic bandwidth selection
method for classification and regression trees
(CART)~\citep{ChGeMc98}.  Bayesian trees have been expanded to include
linear models~\citep{ChGeMc02} and GPs~\citep{GrLe08} in the leaf
nodes.



In a regression setting, the Dirichlet process has been mainly used
for problems with a continuous response.  \citet{WeMuEs94,EsWe95} and
\citet{MuErWe96} used joint Gaussian mixtures for the covariates and
response, and \citet{RoDuGe09} generalized this method using dependent
DPs for multiple response functionals.  However, the method of
\citet{RoDuGe09} can be slow if a fully populated covariance matrix is
used, and is potentially inaccurate if it is assumed diagonal.  To
avoid these issues---which amount to over-fitting the covariate
distribution and under-fitting the response---some researchers have
developed methods that use local weights on the covariates to produce
local response DPs.  This has been achieved with kernels and basis
functions~\citep{GrSt07,DuPiPa07}, GPs~\citep{GeKoMa05} and general
spatial-based weights~\citep{GrSt06,GrSt07,DuGuGe07}.  Still other
methods, again based on dependent DPs, capture similarities between
clusters, covariates or groups of outcomes, including in non-continuous settings~\citep{DeMuRo04,RoDuGe09}.
The method presented here is equally applicable to the continuous
response setting and tries to balance its fit of the covariate and
response distributions by introducing local GLMs---the clustering
structure is based on both the covariates and how the response varies
with them.

There is somewhat less research that develops Bayesian nonparametric
models for other types of response. \citet{MuGe97} and \citet{IbKl98}
used a DP prior for the random effects portion of a GLM.  Likewise,
\citet{AmGhGh03} used a DP prior to model arbitrary symmetric error
distributions in a semi-parametric linear regression model.  While
these are powerful extensions of regression models, they still
maintain the assumption that all covariates enter the model linearly
in the same way.  Our work is closest to \citet{ShNe07}.  They
proposed a model that mixes over both the covariates and response,
where the response is drawn from a multinomial logistic model.  The
DP-GLM studied here is a generalization of their idea.



Finally, asymptotic properties of Dirichlet process regression models
have not been well studied.  Most current literature centers around
consistency of the posterior density for DP Gaussian mixture models
\citep{BaScWa99,GhGhRa99,GhRa03,Wa04,To06} and semi-parametric linear
regression models \citep{AmGhGh03,To06}.  Only recently have the
posterior properties of DP regression estimators been studied.
\citet{RoDuGe09} showed point-wise asymptotic unbiasedness for their
model, which uses a dependent Dirichlet process prior, assuming
continuous covariates under different treatments with a continuous
responses and a conjugate base measure (normal-inverse Wishart).  
In Section \ref{sec:theory} we show pointwise asymptotic unbiasedness of the DP-GLM in both the continuous and categorical response settings.  In the continuous response setting, our results generalize those of \citet{RoDuGe09} and \citet{Ro09}.  Moreover in the categorical response setting, the same theoretical framework provides the same consistency results for the classification model of \citet{ShNe07}.

\section{Mathematical background}
\label{sec:mathReview}

In this section we provide some mathematical background.  We review Dirichlet
process mixture models and generalized linear models.

\paragraph{Dirichlet Process Mixture Models.}
\label{sec:mixture}

The \textit{Dirichlet process} (DP) is a distribution over
distributions~\citep{Fe73}.  It is denoted,
\begin{equation}
  G \sim \textrm{DP}(\alpha G_0),
\end{equation}
where $G$ is a random distribution.  There are two parameters.  The
base distribution $G_0$ is a distribution over the same space as $G$,
e.g., if we want $G$ to be a distribution on reals then $G_0$ must be
a distribution on reals too.  The concentration parameter $\alpha$ is
a positive scalar.  An important property of the DP is that random
distributions $G$ are discrete, and each places its mass on a
countably infinite collection of atoms drawn from $G_0$.


Consider the model
\begin{eqnarray}
  G &\sim& \textrm{DP}(\alpha, G_0) \\
  \theta_i &\sim& G
\end{eqnarray}
The joint distribution of $n$ replicates of $\theta_i$ is
\begin{equation}
  p(\theta_{1:n} \g \alpha, G_0) =
  \int \left(\prod_{i=1}^{n} G(\theta_i) \right) P(G) dG
\end{equation}
One write this joint in a simpler form.  Specifically, the conditional
distribution of $\theta_{n}$ given $\theta_{1:(n-1)}$ follows a Polya
urn distribution~\citep{BlMa73},
\begin{equation}\label{eq:polya}
  \theta_{n} | \theta_{1:(n-1)} \sim \frac{1}{\alpha + n - 1} \sum_{i=1}^{n-1}
  \delta_{\theta_i} + \frac{\alpha}{\alpha + n - 1} \G_0.
\end{equation}
With the chain rule, this specifies the full joint distribution of
$\theta_{1:n}$.

Equation (\ref{eq:polya}) reveals the \textit{clustering property} of
the joint distribution of $\theta_{1:n}$: There is a positive
probability that each $\theta_i$ will take on the value of another
$\theta_j$, leading some of the draws to share values.  This equation
also makes clear the roles of scaling parameter $\alpha$ and base
distribution $\G_0$.  The unique values contained in $\theta_{1:n}$
are drawn independently from $\G_0$ and the parameter $\alpha$
determines how likely $\theta_{n+1}$ is to be a newly drawn value from
$\G_0$ rather than take on one of the values from $\theta_{1:n}$.  The
base measure $\G_0$ controls the distribution of a newly drawn
value.\footnote{Technically, if $G_0$ is itself discrete then the
  ``unique'' values can themselves share values.}

In a DP mixture, $\theta$ is a latent parameter to an observed data
point $x$~\citep{An74},
\begin{align}\notag
  P &\sim \textrm{DP}(\alpha \G_0), \\\notag
  \Theta_i &\sim  P, \\\notag
  x_i|\theta_i &\sim  f(\cdot \g \theta_i).
\end{align}
Examining the posterior distribution of $\theta_{1:n}$ given $x_{1:n}$
brings out its interpretation as an ``infinite clustering'' model.
Because of the clustering property, observations are grouped by their
shared parameters. Unlike finite clustering models, however, the
number of groups is random and unknown.  Moreover, a new data point
can be assigned to a new cluster that was not previously seen in the
data.

\paragraph{Generalized Linear Models.}\label{sec:glm}

Generalized linear models (GLMs) build on linear regression to provide
a flexible suite of predictive models.  GLMs relate a linear model to
a response via a link function; examples include familiar models like
logistic regression, Poisson regression, and multinomial regression.
See \citet{McNe89} for a full discussion.

GLMs have three components: the conditional probability model for
response $Y$, the linear predictor and the link function.  The
probability model for $Y$, dependent on covariates $X$, is
\begin{equation}
  \notag f(y|\eta) = \exp\left(\frac{ y \eta -
      b(\eta)}{a(\phi)} + c(y,\phi)\right).
\end{equation}
Here the canonical form of the exponential family is given, where $a$,
$b$, and $c$ are known functions specific to the exponential family,
$\phi$ is an arbitrary scale (dispersion) parameter, and $\eta$ is the
canonical parameter.  A linear predictor, $X \beta$, is used to
determine the canonical parameter through a set of transformations.
It can be shown that $b^{\prime}(\eta) = \mu =
\E[Y|X]$~\citep{Br86}. However, we can choose a link function
$g$ such that $\mu = g^{-1}(X\beta),$ which defines $\eta$ in terms of
$X\beta$.  The canonical form is useful for discussion of GLM
properties, but we use the form parameterized by mean $\mu$ in the rest of this paper.  GLMs
are simple and flexible---they are an attractive choice for a local
approximation of a global response function.


\section{Dirichlet process mixtures of generalized linear models}
\label{sec:dpglm}

We now turn to Dirichlet process mixtures of generalized linear models
(DP-GLMs), a Bayesian predictive model that places prior mass on a
large class of response densities.  Given a data set of
covariate-response pairs, we describe Gibbs sampling algorithms for
approximate posterior inference and prediction.  Theoretical
properties of the DP-GLM are developed in \mysec{theory}.

\subsection{Model formulation}

\begin{figure}
\begin{center}
\includegraphics*[width=6in, viewport = 10 220 600 570]{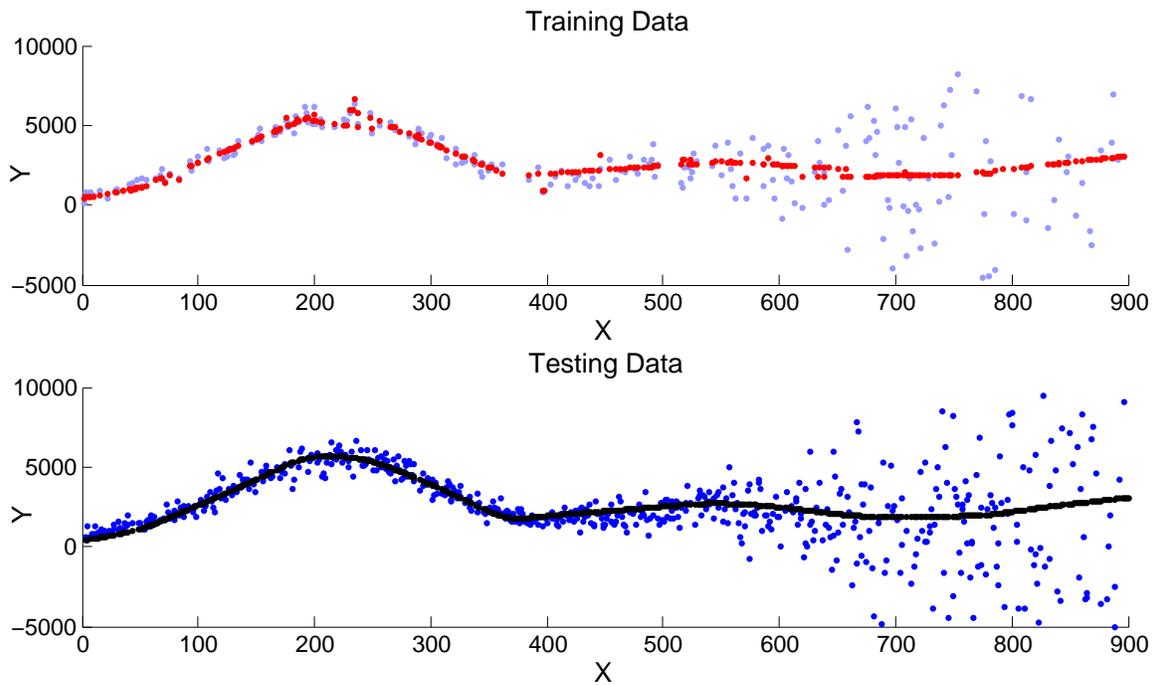}
\caption{The top figure shows the training data (gray) fitted into
  clusters, with the prediction given a single sample from the posterior, $\theta^{(i)}$ (red). The bottom figure shows the smoothed regression estimate (black) for the Gaussian model of Equation (\ref{eq:linear}) with the testing data (blue).  Data plot multipole moments against power spectrum $C_{\ell}$ for cosmic microwave background radiation~\citep{Be03}.}\label{fig:demo}
  \end{center}
\end{figure}

In a DP-GLM, we assume that the covariates $X$ are modeled by a
mixture of exponential-family distributions, the response $Y$ is
modeled by a GLM conditioned on the inputs, and that these models are
connected by associating a set of GLM coefficients with each
exponential family mixture component.  Let $\theta = (\theta_x,
\theta_y)$ denote the bundle of parameters over $X$ and $Y \g X$, and
let $\G_0$ denote a base measure on the space of both.  For example,
$\theta_x$ might be a set of $d$-dimensional multivariate Gaussian
location and scale parameters for a vector of continuous covariates;
$\theta_y$ might be a $d+2$-vector of reals for their corresponding
GLM linear prediction coefficients, along with a GLM dispersion
parameter.  The full model is
\begin{align}
  \notag
  P & \sim DP(\alpha\G_0),\\\notag
  \theta = (\theta_{i,x}, \theta_{y,i}) | P & \sim P,\\\notag
  X_i|\theta_{i,x} & \sim f_{x}(\cdot|\theta_{i,x}), \\\notag
  Y_i|x_i,\theta_{i,y} & \sim GLM(\cdot|X_i,\theta_{i,y}).
\end{align}
The density $f_x$ describes the covariate distribution; the GLM for
$y$ depends on the form of the response (continuous, count, category,
or others) and how the response relates to the covariates (i.e., the
link function).

The Dirichlet process clusters the covariate-response pairs $(x,y)$.
When both are observed, i.e., in ``training,'' the posterior
distribution of this model will cluster data points according to
near-by covariates that exhibit the same kind of relationship to their
response.  When the response is not observed, its predictive
expectation can be understood by clustering the covariates based on
the training data, and then predicting the response according to the
GLM associated with the covariates' cluster.  The DP prior acts as a
kernel for the covariates; instead of being a Euclidean metric, the DP
measures the distance between two points by the probability that the
hidden parameter is shared.  See Figure \ref{fig:demo} for a
demonstration of the DP-GLM.

We now give a few examples of the DP-GLM that will be used throughout this paper.

\paragraph{Example: Gaussian Model.}  We now give an example of the
DP-GLM for continuous covariates/response that will be used throughout
the rest of the paper.  For continuous covariates/response in $\R$, we
model locally with a Gaussian distribution for the covariates and a
linear regression model for the response.  The covariates have mean
$\mu_{i,j}$ and variance $\sigma_{i,j}^2$ for the $j^{th}$ dimension
of the $i^{th}$ observation; the covariance matrix is diagonal for
simplicity.  The GLM parameters are the linear predictor
$\beta_{i,0},\dots,\beta_{i,d}$ and the response variance
$\sigma_{i,y}^2.$ Here, $\theta_{x,i} = (\mu_{i,1:d},\sigma_{i,1:d})$
and $\theta_{y,i} = (\beta_{i,0:d},\sigma_{i,y}).$ This produces a
mixture of multivariate Gaussians.  The full model is,
\begin{align}\label{eq:linear}
  P & \sim DP(\alpha\G_0),\\\notag
  \theta_i | P& \sim P, \\\notag
  X_{i,j} | \theta_{i,x}&\sim N\left(\mu_{ij}, \sigma_{ij}^2\right), &  j = 1,\dots,d,\\\notag
  Y_i | X_i, \theta_{i,y} & \sim N\left(\beta_{i0} + \sum_{j=1}^d \beta_{ij}
  X_{ij}, \sigma_{iy}^2\right).
\end{align}




\paragraph{Example: Multinomial Model~\citep{ShNe07}.}This model was proposed by \citet{ShNe07} for nonlinear classification, using a Gaussian mixture to model continuous covariates and a multinomial logistic model for a categorical response with $K$ categories.  The covariates have mean
$\mu_{i,j}$ and variance $\sigma_{i,j}^2$ for the $j^{th}$ dimension
of the $i^{th}$ observation; the covariance matrix is diagonal for
simplicity.  The GLM parameters are the $K$ linear predictor
$\beta_{i,0,k},\dots,\beta_{i,d,k},$ $k = 1,\dots,K$.  The full model is,
\begin{align}\label{eq:multinomial}
  P & \sim DP(\alpha\G_0),\\\notag
  \theta_i | P& \sim P, \\\notag
  X_{i,j} | \theta_{i,x}&\sim N\left(\mu_{ij}, \sigma_{ij}^2\right), & j = 1,\dots,d,\\\notag
  \p(Y_i=k | X_i, \theta_{i,y}) & = \frac{\exp\left(\beta_{i,0,k}+\sum_{j=1}^d\beta_{i,j,k} X_{i,j}\right)}{\sum_{\ell=1}^K \exp\left(\beta_{i,0,\ell}+\sum_{j=1}^d\beta_{i,j,\ell} X_{i,j}\right)}, & k = 1,\dots,K.
\end{align}

\paragraph{Example: Poisson Model with Categorical Covariates.}The categorical covariates are modeled by a mixture of multinomial distributions and the count response by a Poisson distribution.  If covariate $j$ has $K$ categories, let $(p_{i,j,1},\dots,p_{i,j,K})$ be the probabilities for categories $1,\dots,K$.  The covariates are then turned into indicator variables, $\1_{\{X_{i,j}=k\}}$, which are used with the linear predictor, $\beta_i,0, \beta_{i,1,1:K},\dots, \beta_{i,d,1:K}.$  The full model is,
\begin{align}\label{eq:poisson}
  P & \sim DP(\alpha\G_0),\\\notag
  \theta_i | P& \sim P, \\\notag
  \p(X_{i,j}= k | \theta_{i,x})&= p_{i,j,k}, & j = 1,\dots,d, \ k = 1,\dots,K,\\\notag
  \lambda_i | X_i, \theta_{i,y} & = \exp\left( \beta_{i,0} + \sum_{j=1}^d \sum_{k=1}^K \beta_{i,j,k} \1_{\{X_{i,j}=k\}}\right),\\\notag
  \p(Y_i=k | X_i, \theta_{i,y}) & = \frac{e^{-\lambda_i}\lambda_i^k}{\ell!}, & k = 0, 1, 2, \dots.
\end{align}
We apply Model (\ref{eq:poisson}) to data in Section \ref{sec:numbers}.

\subsection{Heteroscedasticity and overdispersion}
One advantage of the DP-GLM is that it provides a strategy for handling common problems in predictive modeling.  Many models, such as GLMs and Gaussian processes, make assumptions
about data dispersion and homoscedasticity.  Over-dispersion occurs in
single parameter GLMs when the data variance is larger than the
variance predicted by the model mean.  \citet{MuGe97} have
successfully used DP mixtures over GLM intercept parameters to create
classes of models that include over-dispersion.  The DP-GLM retains
this property, but is not limited to linearity in the covariates.

Homoscedasticity refers to the property of variance that is constant among all
covariate regions; heteroscedasticity is variance that changes with
the covariates.  Models like GLMs and Gaussian processes assume
homoscedasticity and can give poor fits when that assumption is
violated.  However, the DP-GLM can naturally accommodate
heteroscedasticity when multiparameter GLMs are used, such as linear,
gamma and negative binomial regression models.  The mixture model
setting allows the variance parameter to vary between clusters,
creating smoothly transitioning heteroscedastic posterior response
distributions.

A demonstration of this property is shown in Figure
\ref{fig:cmbHetero}, where the DP-GLM is compared against a
homoscedastic model, Gaussian processes, and heteroscedastic
modifications of homoscedastic models, treed Gaussian processes and
treed linear models.  The DP-GLM is robust to heteroscedastic
data---it provides a smooth mean function estimate, while the other
models are not as robust or provide non-smooth estimates.

\begin{figure}
\centering
\includegraphics[width=6in]{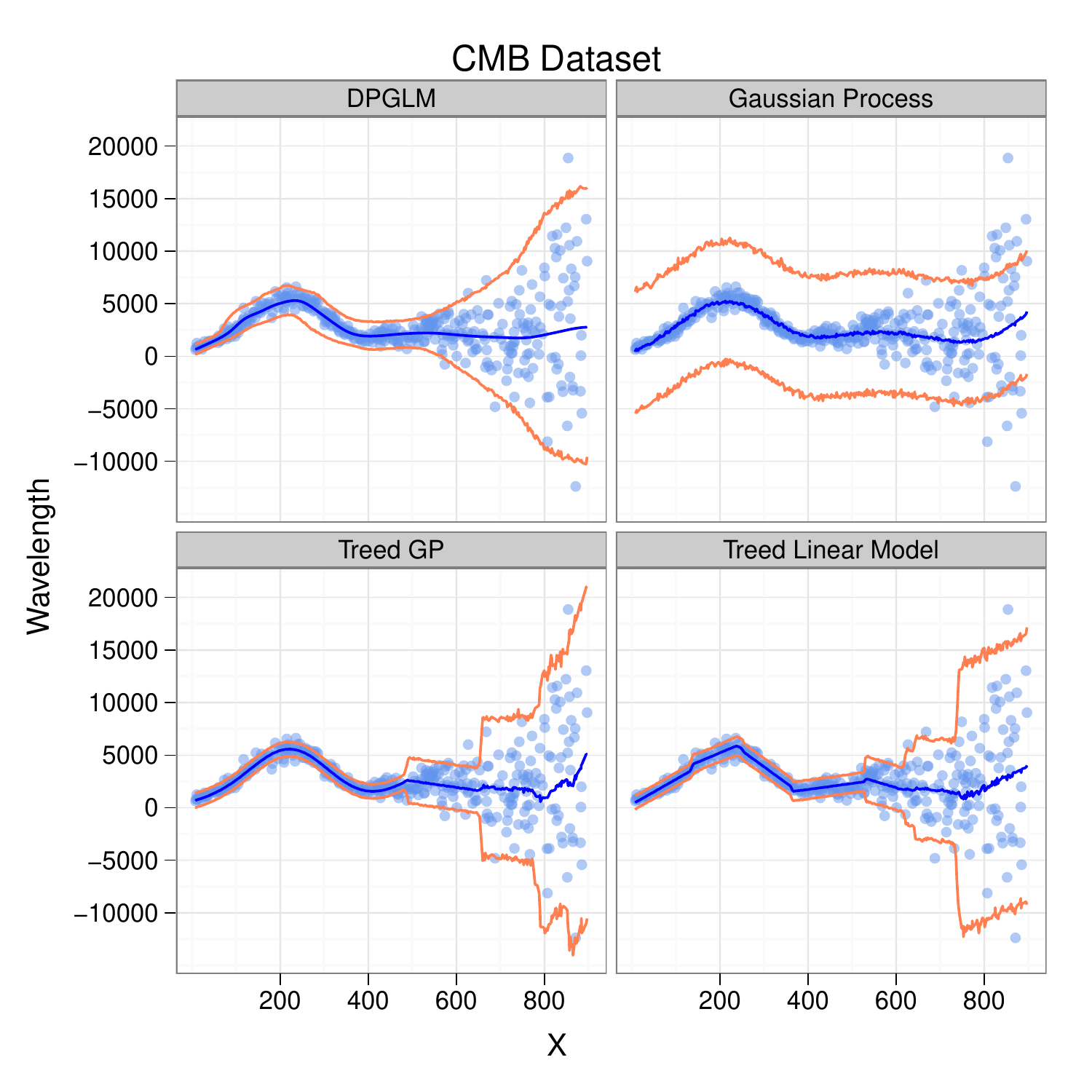}
\caption{Modeling heteroscedasticity with the DP-GLM and other
  Bayesian nonparametric methods.  The estimated mean function is
  given along with a 90\% predicted confidence interval for the
  estimated underlying distribution.  DP-GLM produces a smooth mean
  function and confidence interval.}
\label{fig:cmbHetero}
\end{figure}

\subsection{Posterior prediction with a DP-GLM }

The DP-GLM is used in prediction problems.  Given a collection of
covariate-response pairs $D = (X_i, Y_i)_{i=1}^n$, our goal is to
compute the expected response for a new set of covariates $x$, $\E[Y
\g x, D]$.  We give the step-by-step process for formulating the model
and forming the prediction.

\paragraph{Choosing the mixture component and GLM.}  We begin by
choosing $f_x$ and the GLM.  The Dirichlet process mixture model and GLM provide flexibility both
in terms of the covariates and the response.  Dirichlet process
mixture models allow nearly any type of variable to be modeled within
the covariate mixture and subsequently transformed for use as a
covariate in the GLM.

Note that certain mixture distributions simply support certain types
of covariates but may not necessarily be a good fit.  For example, one covariate might be strictly positive and continuous.  This could be modeled with an exponential mixture,
\begin{equation}\notag
X_i | \lambda_i \sim Exp(\lambda_i).
\end{equation}However, although exponential mixtures have support on $\R_{++}$, they are mixtures of a single parameter exponential family.  This has implications for the variance of the distribution, which is determined by the mean of each component.  Therefore, a mixture of gamma distributions or possibly even Gaussians would be a better fit.  Both the gamma and Gaussian distributions have a mean and dispersion parameter, which free the mixture variance from the mean.

The GLM---for the
conditional response---is chosen in much the same way.

\paragraph{Choosing the base measure and other hyperparameters.} The
choice of the base measure $\G_0$ affects how expressive the DP-GLM
is, the computational efficiency of the prediction and whether some
theoretical properties, such as asymptotic unbiasedness, hold.  For
example, $\G_0$ for the Gaussian model is a distribution over $(\mu_i,
\sigma_i, \beta_{i,0:d}, \sigma_{i,y})$.  A conjugate base measure is
normal-inverse-gamma for each covariate dimension and multivariate
normal inverse-gamma for the response parameters.  This $\G_0$ allows
all continuous, integrable distributions to be supported, retains
theoretical properties, such as asymptotic unbiasedness, and yields highly efficient posterior sampling by allowing the Gibbs sampler to be collapsed~\citep{Ne00}.
However, this base measure is not expressive for small amounts of data. The tails
quickly decline and the mean is tied to the variance.  In summary, the
base measure is often chosen in accordance to data size, distribution
type, distribution features (heterogeneity, etc) and computational
constraints.



Hyperparameters for the DP-GLM include the DP scaling parameter
$\alpha$ and hyperparameters parameters for the base measure $\G_0$.
It is often useful to place a gamma prior on $\alpha$~\citep{EsWe95}, while the
parameters for $\G_0$ may have their own prior as well.  Each level of
priors reduces their influence but adds computational complexity~\citep{EsWe95}.


\paragraph{Approximating the posterior and forming predictions.}  Our
ultimate goal is to form a conditional expectation of the response,
given a new set of covariates $x$ and the observed data $D$, $\E[Y \g
X = x, D]$.  Following the Bayesian regression methodology, we use
iterated expectation, conditioning on the latent variables,
\begin{equation}\label{eq:expectationTower}
\E\left[Y\g X = x, D \right] =
\E\left[ \E \left[ Y \g X = x, \theta_{1:n} \right] \g D\right].
\end{equation}
The inner expectation is straight-forward to compute.  Conditional on
the latent parameters $\theta_{1:n}$ that generated the observed data,
the expectation of the response is
\begin{equation}\label{eq:reg0}
  \E[Y|X = x,\theta_{1:n}] = \frac{\alpha \int_{\T} \E\left[Y|X = x,\theta\right]f_x(x|\theta) \G_0(d\theta) + \sum_{i=1}^n
    \E\left[Y|X = x,\theta_{i}\right]f_x(x|\theta_{i})}{\alpha
    \int_{\T} f_x(x|\theta) \G_0(d\theta) + \sum_{i=1}^n
    f_x(x|\theta_i)}.
\end{equation}Since $Y$ is assumed to be a GLM, the quantity
$\E\left[Y|X = x,\theta\right]$ is analytically available as a function of
$x$ and $\theta$.

The outer expectation of Equation (\ref{eq:expectationTower}) is
generally intractable.  We approximate it by Monte Carlo integration
using $M$ posterior samples of $\theta_{1:n}$,
\begin{equation}\label{eq:intApprox}
  \E \left[Y \g X =x, D\right] \approx \frac{1}{M} \sum_{m=1}^M \E
  \left[ Y \g X = x, \theta_{1:n}^{(m)}\right].
\end{equation}
The observations $\theta_{1:n}^{(m)}$ are i.i.d. from the posterior
distribution of $\theta_{1:n}\g D$.

We use Markov chain Monte Carlo (MCMC) to obtain $M$ i.i.d. samples
from this distribution.  Specifically, we use Gibbs sampling, which is
an effective algorithm for DP mixture models.  (See \citet{Es94},
\citet{Ma94}, \citet{EsWe95} and \citet{MaMu98} for foundational work;
\citet{Ne00} provides a modern treatment and state of the art
algorithms.)  In short, we construct a Markov chain on the hidden
variables $\theta_{1:n}$ such that its limiting distribution is
the posterior of interest.  Details for its implementation are given in Appendix \ref{sec:inference}.



\subsection{Comparison to the Dirichlet process mixture model
  regression}\label{sec:dpRegular}

The DP-GLM directly models $Y$ as being conditioned on $X$.  Modeling the
joint distribution of $(x,y)$ as coming from a common mixture component
in a classical DP mixture (see Section \ref{sec:mathReview}) also induces a conditional
distribution of $Y$ given $X$.  In this Section, we conceptually compare
these two approaches.  (They are compared empirically in this section and Section \ref{sec:numbers}.)

A generic Dirichlet process mixture model (DPMM) has the form,
\begin{align}\label{eq:dpmm}
P & \sim DP(\alpha \G_0),\\\notag
\theta_i | P & \sim P,\\\notag
X_i | \theta_{i,x} & \sim f_x(x|\theta_{i,x}),\\\notag
Y_i | \theta_{i,y} & \sim f_y(y|\theta_{i,y}).
\end{align}DPMMs with this form have been studied for regression~\citep{EsWe95}, but have generally not been used in practice due to poor results (with a diagonal covariance matrix) or computational difficulties (with a full covariance matrix).  We focus on the former case, with diagonal covariance, to study why it has poor results and how the DP-GLM improves on these with a minimal increase in computational difficulty.  The difference between Model (\ref{eq:dpmm}) and the DP-GLM is that
the distribution of $Y$ given $\theta$ is conditionally independent of
the covariates $X$.  This is a small difference that has large
consequences for the posterior distribution and predictive results.

Consider the log-likelihood of the posterior of the DPMM given in Model (\ref{eq:dpmm}).
Assume that $f_y$ is a single parameter exponential, where $\theta_y =
\beta$,
\begin{equation}\label{eq:dpmmPost}
  \ell(\theta^{dp} \g D) \propto \sum_{i=1}^K \left[\ell(\beta_{C_i})
    + \sum_{c \in C_i} \ell(y_c \g \beta_{C_i})
    +\sum_{j=1}^d \ell(\theta_{C_i,x_j} \g D) \right].
\end{equation}
The log-likelihood of the DP-GLM posterior for a single parameter
exponential family GLM, where $\theta_y = (\beta_0,\dots,\beta_d)$, has the form,
\begin{equation}\label{eq:dpglmPost}
  \ell(\theta^{dpglm} \g D)  \propto \sum_{i=1}^K\left[ \sum_{j=0}^d
    \ell(\beta_{C_i,j}) +\sum_{c \in C_i} \ell(y_c \g \beta_{C_i}^T
    x_c)+ \sum_{j=1}^d \ell(\theta_{C_i,x_j}\g D)\right].
\end{equation}
As the number of covariates grows, the likelihood associated with the
covariates grows in both equations.  However, the likelihood
associated with the response also grows with the extra response
parameters in Equation (\ref{eq:dpglmPost}), whereas it is fixed in
Equation (\ref{eq:dpmmPost}).

These posterior differences lead to two predictive differences: 1) the
DP-GLM is much more resistant to dimensionality than the DPMM, and 2)
as the dimensionality grows, the DP-GLM produces less stable
predictions than the DPMM.  Since the number of response related
parameters grows with the number of covariate dimensions in the
DP-GLM, the relative posterior weight of the response does not shrink
as quickly in the DP-GLM as it does in the DPMM.  This keeps the
response variable relatively important in the selection of the mixture
components and hence makes the DP-GLM a better predictor than the DPMM
as the number of dimensions grows.  However, each additional $\beta_j$ adds some noise, so when $d$ is large the DP-GLM estimate tends to be noisy.

While the additional GLM parameters help maintain the relevance of the
response, they also add noise to the prediction.  This can be seen in
Figure \ref{fig:dimCompare}.  The GLM parameters in this figure have a
Gaussian base measure, effectively creating a local ridge
regression.\footnote{In unpublished results, we have also tried other
  sparsity-inducing base measures, such as a Laplacian distribution
  with an L1 penalty.  They produced less stable results than the
  Gaussian base measure, likely due to the sample size differences
  between the clusters.}  In lower dimensions, the DP-GLM produced
more stable results than the DPMM because a smaller number of larger
clusters were required to fit the data well.  The DPMM, however,
consistently produced stable results in higher dimensions as the
response became more of a sample average than a local average.  The
DPMM has the potential to predict well if changes in the mean function
coincide with underlying local modes of the covariate density.  However, the DP-GLM forces the covariates into clusters that coincide more with the response variable due to the inclusion of the slope parameters.


\begin{figure}
\centering
\includegraphics[width=6in]{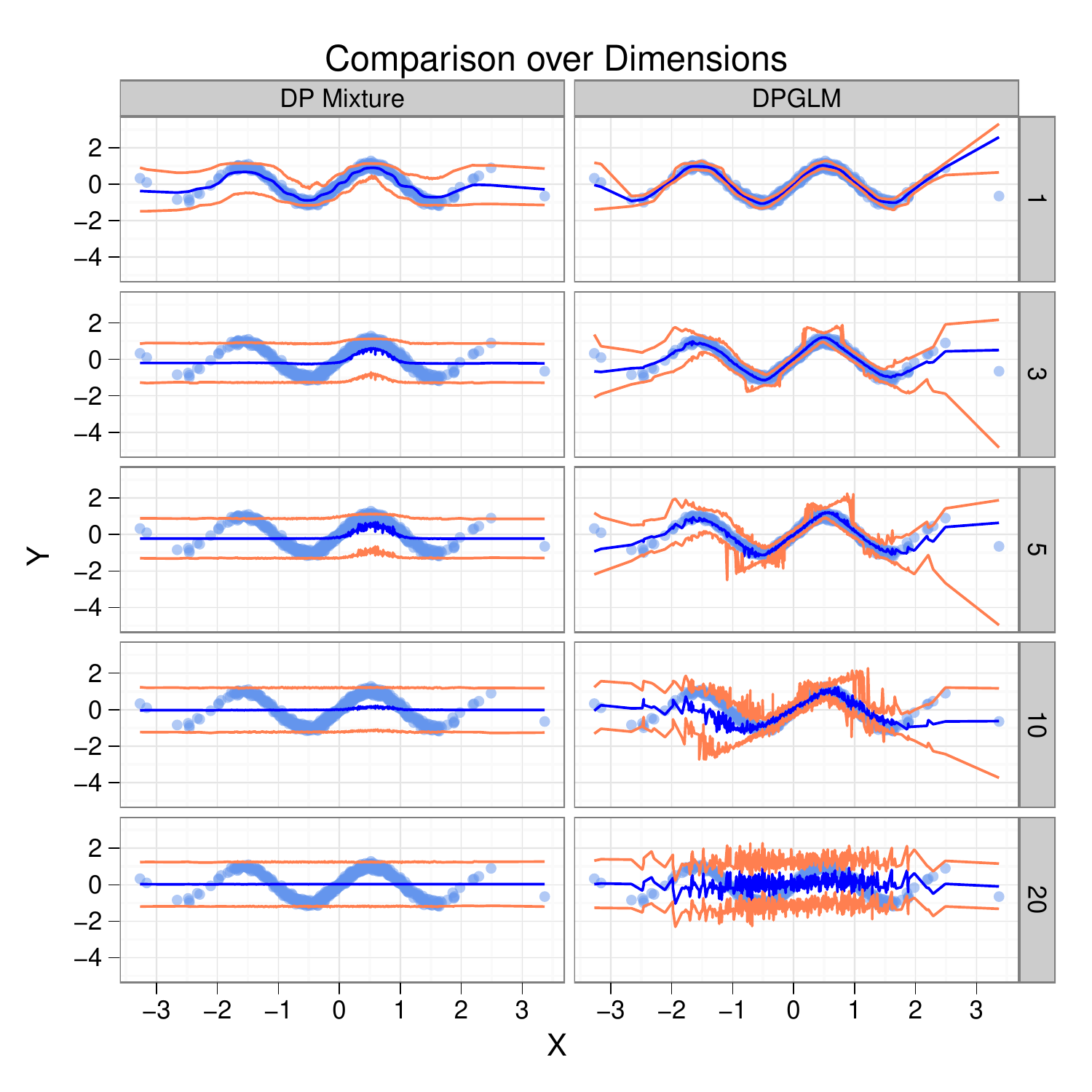}
\caption{A plain Dirichlet process mixture model regression (left)
  versus DP-GLM, plotted against the number of spurious dimensions
  (vertical plots).  The estimated mean function is given along with a
  90\% predicted confidence interval for the estimated underlying
  distribution.  Data have one predictive covariate and a varying
  number of spurious covariates.  The covariate data were generated by
  a mixture model.  DP-GLM produces a smoother mean function and is
  much more resistant to spurious dimensionality.}
\label{fig:dimCompare}
\end{figure}

We now discuss the theoretical properties of the DP-GLM.

\section{Asymptotic Unbiasedness of the DP-GLM Regression
  Model}\label{sec:theory}



A desirable property of any estimator is that it should be unbiased,
particularly in the limit.  \citet{DiFr86} gives an example of a
location model with a Dirichlet process prior where the estimated
location can be bounded away from the true location, even when the
number of observations approaches infinity. We want to assure that
DP-GLM does not end up in a similar position.

Notation for this section is more complicated than the notation for
the model.  Let $f_0(x,y)$ be the true joint distribution of $(x,y)$;
in this case, we will assume that $f_0$ is a density.  Let
$\mathcal{F}$ be the set of all density functions over $(x,y)$.  Let
$\Pi^f$ be the prior over $\mathcal{F}$ induced by the DP-GLM model.
Let $\E_{f_0}[\cdot]$ denote the expectation under the true
distribution and $\E_{\Pi^f}[\cdot]$ be the expectation under the
prior $\Pi^f$.

In general, an estimator is a function of observations.  Assuming a
true distribution of those observations, an estimator is called
unbiased if its expectation under that distribution is equal to the
value that it estimates.  In the case of DP-GLM, that would mean for every $x
\in \mathcal{A}$ and every $n > 0$,
\begin{equation}\notag
  \E_{f_0}\left[\E_{\Pi}[Y | x, (X_i,Y_i)_{i=1}^n]\right] = \E_{f_0}[Y|x],
\end{equation}where $\mathcal{A}$ is some fixed domain, $\E_{\Pi}$ is
the expectation with respect to the prior $\Pi$ and $\E_{f_0}$ is the
expectation with respect to the true distribution.

Since we use Bayesian priors in DP-GLM, we will have bias in almost
all cases~\citep{GeCaSt04}. The best we can hope for is asymptotic
unbiasedness, where as the number of observations grows to infinity,
the mean function estimate converges to the true mean function. That
is, for every $x \in \mathcal{A}$,
\begin{equation}\notag
  \E_{\Pi}[Y | x, (X_i,Y_i)_{i=1}^n] \rightarrow \E[Y|x]\ \ \ \
  \textrm{ as }n \rightarrow \infty.
\end{equation}\citet{DiFr86} give an example for a location problem
with a DP prior where the posterior estimate was not asymptotically
unbiased.  Extending that example, it follows that estimators with DP
priors do not automatically receive asymptotic unbiasedness.  We use
consistency and uniform integrability to give conditions for
asymptotic unbiasedness.

Consistency, the notion that as the number of observations goes to
infinity the posterior distribution accumulates in neighborhoods
arbitrarily ``close'' to the true distribution, is tightly related to
both asymptotic unbiasedness and mean function estimate
existence. Weak consistency assures that the posterior distribution
accumulates in regions of densities where ``properly behaved''
functions (i.e., bounded and continuous) integrated with respect to
the densities in the region are arbitrarily close to the integral with
respect to the true density.  The expectation may not be bounded;
in addition to weak consistency, uniform integrability is needed to
guarantee that the posterior expectation converges to the true
expectation, giving asymptotic unbiasedness.  Uniform integrability
also ensures that the posterior expectation almost surely exists with
every additional observation.  Therefore we need to show weak
consistency and uniform integrability.

\subsection{Asymptotic Unbiasedness}\label{subsec:generalThm}
We approach asymptotic unbiasedness by showing weak consistency for the posterior of the \textit{joint} distribution and then using uniform integrability to show that the conditional expectation $\E_{\Pi^f}[Y\g X = x, (X_i,Y_i)_{i=1}^n]$ exists for every $n$ and converges to the true expectation almost surely.  The proof for this is a bit involved, so it has been placed in the Appendix.

\begin{theorem}\label{thm:generalConvergence}
Let $x$ be in a compact set $\mathcal{C}$ and $\Pi^f$ be a prior on $\mathcal{F}$. If,
\begin{enumerate}[(i)]
\item for every $\delta > 0$, $\Pi^f$ puts positive measure on 
\begin{equation}\notag
\left\{f: \int f_0(x,y) \log \frac{f_0(x,y)}{f(x,y)} dxdy < \delta\right\},\end{equation}
\item $\int |y| f_0(y|x)dy < \infty$ for every $x \in \mathcal{C}$, and
\item there exists an $\epsilon > 0$ such that for every $x\in \mathcal{C}$,
\begin{equation}\notag
\int \int  |y|^{1+\epsilon} f_y(y|x,\theta) \G_0(d\theta) < \infty,
\end{equation}
\end{enumerate}
then for every $n \geq 0$, $\E_{\Pi}[Y|x,(X_i,Y_i)_{1:n}]$ exists and has the limit $\E_{f_0} [Y|X =x]$, almost surely with respect to the observation product measure, $\p_{f_0^{\infty}}$.
\end{theorem}
The conditions of Theorem \ref{thm:generalConvergence} must be checked
for the problem ($f_0$) and prior ($\Pi^f$) pair, and can be difficult to show.  Condition $(i)$ assures weak consistency of the posterior, condition $(ii)$ guarantees a mean function exists in the limit and condition $(iii)$ guarantees that positive probability is only placed on densities that yield a finite mean function estimate, which is used for uniform integrability.  Condition $(i)$ is quite hard to show, but there has been recent work demonstrating weak consistency for a number of Dirichlet process mixture models and priors~\citep{GhGhRa99,GhRa03,AmGhGh03,To06}.  See the Appendix \ref{sec:main} for a proof of Theorem \ref{thm:generalConvergence}.

\subsection{Asymptotic Unbiasedness Example: Gaussian Model}\label{subsec:example}
Theorem \ref{thm:continuous} gives conditions for when Theorem \ref{thm:generalConvergence} holds for the Gaussian model.  We will then given some examples of when Theorem \ref{thm:continuous} holds.
\begin{theorem}\label{thm:continuous}
Suppose that:
\begin{enumerate}[(i)]
 \item $f_0(x,y)$ is absolutely continuous with respect to $\R^{d+1}$ and has compact support,
 \item all location-scale-slope parameter mixtures are in the weak support of $\G_0$, and
 \item $\G_0$ satisfies assumption $(iii)$ of Theorem \ref{thm:generalConvergence}.
\end{enumerate}Then, the conclusions of Theorem \ref{thm:generalConvergence} hold.
\end{theorem}
See the Appendix \ref{sec:gaussian} for the proof.  Examples that satisfy Theorem \ref{thm:continuous} are as follows, with technical results in the Appendix.
\paragraph{Normal-Inverse-Wishart.} Note that in the Gaussian case, slope parameters can be generated by a full covariance matrix: using a conjugate prior, a Normal-Inverse-Wishart, will produce an instance of the DP-GLM.  Define the following model, which was used by \citet{MuErWe96},
\begin{align}\label{eq:inverseWish}
P & \sim DP(\alpha \G_0),\\\notag
\theta_i \g P & \sim P,\\\notag
(X_i,Y_i) \g \theta_i & \sim N(\mu, \Sigma).
\end{align}The last line of Model (\ref{eq:inverseWish}) can be broken down in the following manner,
\begin{align}\notag
X_i \g \theta_i & \sim N\left(\mu_x, \Sigma_x\right),\\\notag
Y_i \g \theta_i & \sim N\left(\mu_y + b^T \Sigma_x^{-1} b(X_i-\mu_x), \sigma_y^2 - b^T \Sigma_x^{-1} b\right),
\end{align}where
\begin{align}\notag
\mu & = 
\left[
\begin{array}{c}
  \mu_y \\
  \mu_x 
\end{array}
\right], & \Sigma & = 
\left[
\begin{array}{ccc}
 \sigma_y^2 &   b^T \\
  b & \Sigma_x     
\end{array}
\right].
\end{align}
We can then define $\beta$ as,
\begin{align}\notag
\beta_0 & = \mu_y -b^T \Sigma_x^{-1}\mu_x, & \beta_{1:d} & = b^T \Sigma_x^{-1}.
\end{align}
The base measure $\G_0$ is defined as,
\begin{equation}\notag
(\mu,\Sigma) \sim Normal \ Inverse \ Wishart(\mathbf{\lambda},\nu,a,B).
\end{equation}Here $\lambda$ is a mean vector, $\nu$ is a scaling parameter for the mean, $a$ is a scaling parameter for the covariance, and $B$ is a covariance matrix.

\paragraph{Diagonal Normal-Inverse-Gamma.}It is often more computationally efficient to specify that $\Sigma_x$ is a diagonal matrix.  In this case, we can specify a conjugate base measure component by component:
\begin{align}\notag
\sigma_{i,j} & \sim Inverse \ Gamma(a_{j},b_j), & j = 1,\dots,d,\\\notag
\mu_{i,j} \g \sigma_{i,j} & \sim N(\lambda_j,\sigma_{i,j}/\nu_j), & j = 1,\dots,d,\\\notag
\sigma_{i,y} & \sim Inverse \ Gamma(a_y,b_y), \\\notag
\beta_{i,j} \g \sigma_{i,y} & \sim N_{d+1}(\mathbf{\lambda}_y,\sigma_y/\nu_y).
\end{align}The Gibbs sampler can still be collapsed, but the computational cost is much lower than the full Normal-Inverse-Wishart.

\paragraph{Normal Mean, Log Normal Variance.}Conjugate base measures tie the mean to the variance and can be a poor fit for small, heteroscedastic data sets.  The following base measure was proposed \citet{ShNe07},
\begin{align}\notag
\log(\sigma_{i,j}) & \sim N(m_{j,\sigma},s^2_{j,\sigma}), & j = y,1,\dots,d,\\\notag
\mu_{i,j} & \sim N(m_{j,\mu},s_{j,\mu}^2), & j = 1,\dots,d,\\\notag
\beta_{i,j} & \sim N(m_{j,\beta},s_{j,\beta}^2) & j = 0,\dots,d.
\end{align}

\subsection{Asymptotic Unbiasedness Example: Multinomial Model}\label{subsec:example2}
Now consider the multinomial model of \citet{ShNe07}, given in Model (\ref{eq:multinomial}),
\begin{align}\notag
  P & \sim DP(\alpha\G_0),\\\notag
  \theta_i | P& \sim P, \\\notag
  X_{i,j} | \theta_{i,x}&\sim N\left(\mu_{ij}, \sigma_{ij}^2\right), & j = 1,\dots,d,\\\notag
  \p(Y_i=k | X_i, \theta_{i,y}) & = \frac{\exp\left(\beta_{i,0,k}+\sum_{j=1}^d\beta_{i,j,k} X_{i,j}\right)}{\sum_{\ell=1}^K \exp\left(\beta_{i,0,\ell}+\sum_{j=1}^d\beta_{i,j,\ell} X_{i,j}\right)}, & k = 1,\dots,K
\end{align} 
Theorem \ref{thm:multinomial} gives conditions for when Theorem \ref{thm:generalConvergence} holds for the Multinomial model.  We will then given some examples of when Theorem \ref{thm:multinomial} holds.
\begin{theorem}\label{thm:multinomial}
Suppose that:
\begin{enumerate}[(i)]
 \item $f_0(x)$ is absolutely continuous with respect to $\R^{d}$ and has compact support,
 \item $\p(Y=k\g x)$ is continuous in $x$ for every $x \in \mathcal{C}$ and $k = 1,\dots, K$, 
 \item all location-scale-slope parameter mixtures are in the weak support of $\G_0$, and
 \item $\G_0$ satisfies assumption $(iii)$ of Theorem \ref{thm:generalConvergence}.
\end{enumerate}Then, the conclusions of Theorem \ref{thm:generalConvergence} hold.
\end{theorem}
See Appendix \ref{sec:multi} for the proof.  We now give some examples of base measures that satisfy Theorem \ref{thm:multinomial}.  In general, the base measure for the continuous covariates is the same as those in the Gaussian model, while the GLM parameters are given a Gaussian base measure.  Technical results are discussed in the Appendix \ref{sec:multi}.

\paragraph{Normal-Inverse-Wishart.}The covariates have a Normal-Inverse-Wishart base measure while the GLM parameters have a Gaussian base measure,
\begin{align}\notag
(\mu_{i,x},\Sigma_{i,x}) & \sim Normal \ Inverse \ Wishart (\mathbf{\lambda},\nu,a,B),\\\notag
\beta_{i,j,k} & \sim N(m_{j,k},s_{j,k}^2), & j = 0,\dots,d, \ \ k = 1,\dots,K.
\end{align}

\paragraph{Diagonal Normal-Inverse-Gamma.}It is often more computationally efficient to specify that $\Sigma_x$ is a diagonal matrix.  Again, we can specify a conjugate base measure component by component while keeping the Gaussian base measure on the GLM components,
\begin{align}\notag
\sigma_{i,j} & \sim Inverse \ Gamma(a_{j},b_j), & j = 1,\dots,d,\\\notag
\mu_{i,j} \g \sigma_{i,j} & \sim N(\lambda_j,\sigma_{i,j}/\nu_j), & j = 1,\dots,d,\\\notag
\beta_{i,j,k} \g \sigma_{i,y} & \sim N(m_{j,k},s_{j,k}^2), & j=0,\dots,d, \ \ k = 1,\dots,K.
\end{align}

\paragraph{Normal Mean, Log Normal Variance.}Likewise, for heteroscedastic covariates we can use the log normal base measure of \citet{ShNe07},
\begin{align}\notag
\log(\sigma_{i,j}) & \sim N(m_{j,\sigma},s^2_{j,\sigma}), & j = 1,\dots,d,\\\notag
\mu_{i,j} & \sim N(m_{j,\mu},s_{j,\mu}^2), & j = 1,\dots,d,\\\notag
\beta_{i,j,k} & \sim N(m_{j,k,\beta},s_{j,k,\beta}^2) & j = 0,\dots,d, \ \ k = 1,\dots, K.
\end{align}

\section{Empirical study}\label{sec:numbers}
We compare the performance of DP-GLM regression to other regression
methods. We studied data sets that illustrate the strengths of the
DP-GLM, including robustness with respect to data type,
heteroscedasticity and higher dimensionality than can be approached
with traditional methods.  \citet{ShNe07} used a similar model on data
with categorical covariates and count responses; their numerical
results were encouraging.  We tested the DP-GLM on the following
datasets.

\paragraph{Datasets.}
We selected three data sets with continuous response variables. They highlight various data difficulties within regression, such as error heteroscedasticity, moderate dimensionality (10--12 covariates), various input types and response types.
\begin{itemize}
\item \textbf{Cosmic Microwave Background (CMB) \cite{Be03}.} The data set consists of 899 observations which map positive integers $\ell = 1,2, \dots, 899$, called `multipole moments,' to the power spectrum $C_{\ell}$. Both the covariate and response are considered continuous. The data pose challenges because they are highly nonlinear and heteroscedastic. Since this data set is only two dimensions, it allows us to easily demonstrate how the various methods approach estimating a mean function while dealing with non-linearity and heteroscedasticity.
\item \textbf{Concrete Compressive Strength (CCS) \cite{Ye98}.} The data set has eight covariates: the components cement, blast furnace slag, fly ash, water, superplasticizer, coarse aggregate and fine aggregate, all measured in $kg$ per $m^3$, and the age of the mixture in days; all are continuous. The response is the compressive strength of the resulting concrete, also continuous. There are 1,030 observations. The data have relatively little noise. Difficulties arise from the moderate dimensionality of the data.
\item \textbf{Solar Flare (Solar) \cite{Br89}.} The response is the number of solar flares in a 24 hour period in a given area; there are 11 categorical covariates. 7 covariates are binary and 4 have 3 to 6 classes for a total of 22 categories. The response is the sum of all types of solar flares for the area. There are 1,389 observations. Difficulties are created by the moderately high dimensionality, categorical covariates and count response. Few regression methods can appropriately model this data.
\end{itemize}Dataset testing sizes ranged from very small (20 observations) to moderate sized (800 observations).  Small dataset sizes were included due to interests in (future) online applications.

\paragraph{Competitors.}
The competitors represent a variety of regression methods; some methods are only suitable for certain types of regression problems.
\begin{itemize}
\item \textbf{Ordinary Least Squares (OLS).} A parametric method that often provides a reasonable fit when there are few observations. Although OLS can be extended for use with any set of basis functions, finding basis functions that span the true function is a difficult task. We naively choose $[1 \, X_1 \, \dots \, X_d]^T$ as basis functions. OLS can be modified to accommodate both continuous and categorical inputs, but it requires a continuous response function.
\item \textbf{CART.} A nonparametric tree regression method generated by the \textsf{Matlab} function \textsf{classregtree}. It accommodates both continuous and categorical inputs and any type of response.
\item \textbf{Bayesian CART.} A tree regression model with a prior over tree size~\citep{ChGeMc98}; it was implemented in \textsf{R} with the \textsf{tgp} package.
\item \textbf{Bayesian Treed Linear Model.} A tree regression model with a prior over tree size and a linear model in each of the leaves~\citep{ChGeMc02}; it was implemented in \textsf{R} with the \textsf{tgp} package.
\item \textbf{Gaussian Processes (GP).} A nonparametric method that can accommodate only continuous inputs and continuous responses. GPs were generated in \textsf{Matlab} by the program \textsf{gpr} of \cite{RaWi06}. It is suitable only for continuous responses and covariates.
\item \textbf{Treed Gaussian Processes.} A tree regression model with a prior over tree size and a GP on each leaf node~\citep{GrLe08}; it was implemented in \textsf{R} with the \textsf{tgp} package.
\item \textbf{Basic DP Regression.} Similar to DP-GLM, except the response is a function only of $\mu_y$, rather than $\beta_0 + \sum \beta_i x_i.$  For the Gaussian model,
\begin{align}\notag
P & \sim DP(\alpha \G_0),\\\notag
\theta_i|P & \sim P,\\\notag
X_i | \theta_i & \sim N(\mu_{i,x},\sigma_{i,x}^2),\\\notag
Y_i|\theta_i & \sim N(\mu_{i,y},\sigma_{i,y}^2.)
\end{align}This model was explored in Section \ref{sec:dpRegular}.
\item \textbf{Poisson GLM (GLM).} A Poisson generalized linear model, used on the Solar Flare data set. It is suitable for count responses.
\end{itemize}

\paragraph{Cosmic Microwave Background (CMB) Results.}
For this dataset, we used a Guassian model with base measure
\begin{align}\notag
\mu_x & \sim N(m_x, s_x^2), & \sigma_x^2 & \sim \exp\left\{ N(m_{x,s}, s_{x,s}^2)\right\},\\\notag
\beta_{0:d} & \sim N(m_{y,0:d}, s_{y,0:d}^2), & \sigma_y^2 & \sim \exp\left\{ N(m_{x,s}, s_{x,s}^2)\right\}.
\end{align}This prior was chosen because the variance tails are heavier than an inverse gamma and the mean is not tied to the variance.  It is a good choice for heterogeneous data because of those features.  Computational details are given in Appendix \ref{sec:cmb}.

All non-linear methods except for CART (DP-GLM, Bayesian CART, treed linear models, GPs and treed GPs) did comparably on this dataset; CART had difficulty finding an appropriate bandwidth.  Linear regression did poorly due to the non-linearity of the dataset.  Fits for heteroscedasticity for the DP-GLM, GPs, treed GPs and treed linear models on 250 training data points can be seen in Figure \ref{fig:cmbHetero}.  See Figure \ref{fig:cmb} and Table \ref{tab:cmb} for results.

\begin{figure}
  \begin{center}
    \includegraphics*[width=6in]{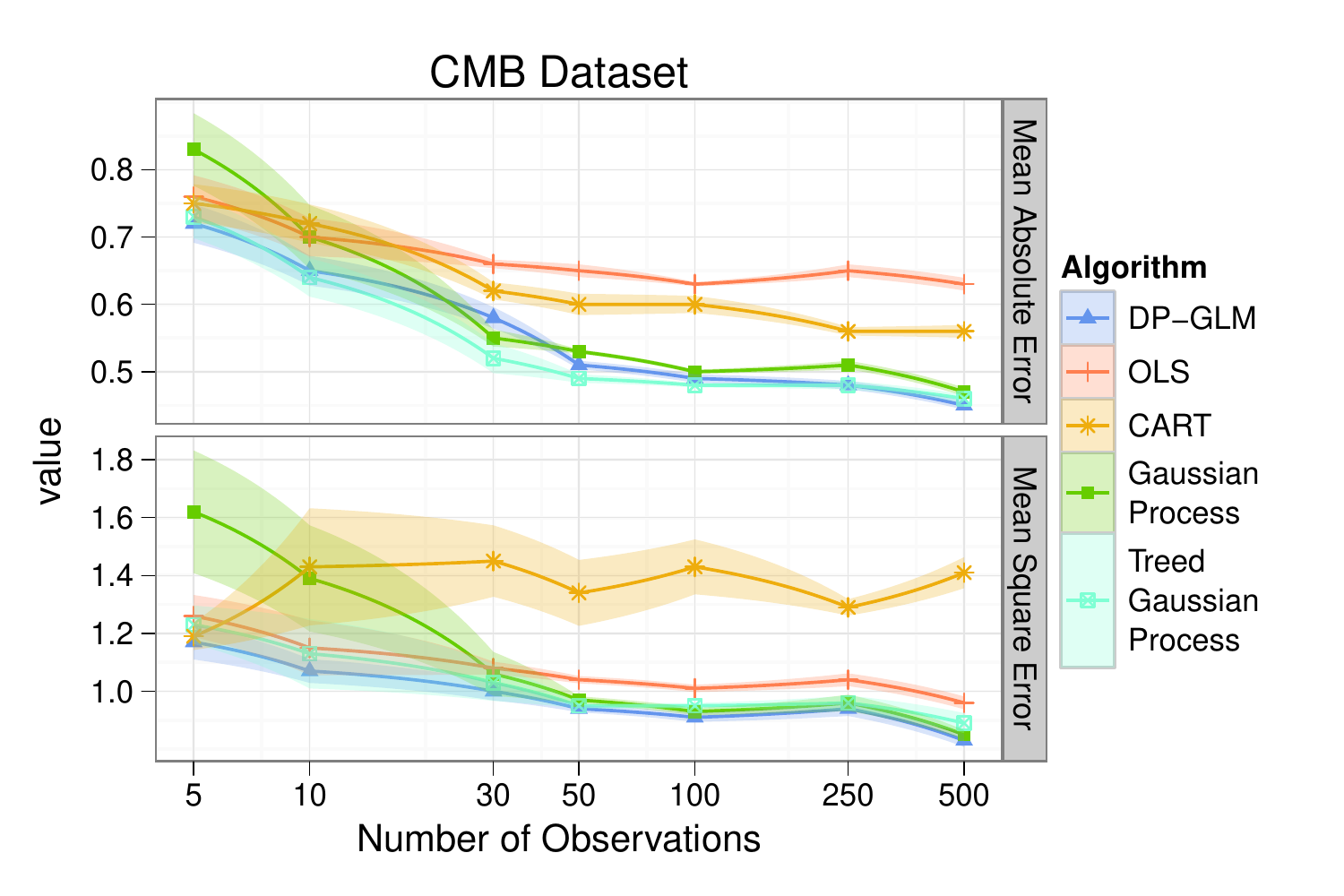}
  \end{center}
  \caption{The average mean absolute error (top) and mean squared
    error (bottom) for ordinary least squares (OLS), tree
    regression, Gaussian processes and DP-GLM on the CMB data
    set. The data were normalized. Mean $+/-$ one standard deviation
    are given for each method.}\label{fig:cmb}
\end{figure}


\begin{table}
\begin{sc}
\small{
\begin{center}
\begin{tabular}{l | c c c c c | c c c c c }\hline
Method & \multicolumn{5}{c|}{Mean Absolute Error}& \multicolumn{5}{c}{Mean Square Error}  \\
\textit{Training set size} & 30 & 50 & 100 & 250 & 500& 30 & 50 & 100 & 250 & 500\\\hline
DP-GLM & 0.58 & 0.51 & 0.49 & {\bf0.48} & {\bf0.45}
& {\bf1.00} & {\bf0.94} & {\bf0.91} & {\bf0.94} & {\bf0.83}\\
Linear Regression & 0.66 & 0.65 & 0.63 & 0.65 & 0.63
& 1.08& 1.04 & 1.01 & 1.04 & 0.96\\
CART & 0.62 & 0.60 & 0.60 & 0.56 & 0.56
& 1.45 & 1.34 & 1.43 & 1.29 & 1.41\\
Bayesian CART & 0.66 & 0.64 & 0.54 & 0.50 & 0.47 
& 1.04 & 1.01 & 0.93 & {\bf0.94} & 0.84\\
Treed Linear Model &0.64 & 0.52 & 0.49 & 0.48 & 0.46 
& 1.10 & 0.95 & 0.93 & 0.95 & 0.85 \\
Gaussian Process & 0.55 & 0.53 & 0.50 & 0.51 & 0.47 
& 1.06 & 0.97 & 0.93 & 0.96 & 0.85\\
Treed GP & {\bf0.52} & {\bf0.49} & {\bf0.48} & {\bf0.48} & 0.46
& 1.03 & 0.95 & 0.95 & 0.96 & 0.89 \\
\hline
\end{tabular}
\end{center}
}\normalsize
\end{sc}
\caption{Mean absolute and square errors for methods on the CMB data set by training data size.  The best results for each size of training data are in bold.}
\label{tab:cmb}
\end{table}

\paragraph{Concrete Compressive Strength (CCS) Results.}
The CCS dataset was chosen because of its moderately high
dimensionality and continuous covariates and response.  For this
dataset, we used a Gaussian model and a conjugate base measure with
conditionally independent covariate and response parameters,
\begin{align}\notag
(\mu_x,\sigma^2_x) & \sim Normal-Inverse-Gamma(m_x,s_x,a_x,b_x),\\\notag
(\beta_{0:d},\sigma_y^2) & \sim Multivariate \ Normal-Inverse-Gamma(M_y, S_y, a_y,b_y).
\end{align}This base measure allows the sampler to be fully collapsed but has fewer covariate-associated parameters than a full Normal-Inverse-Wishart base measure, giving it a better fit in a moderate dimensional setting.  In testing, it also provided better results for this dataset than the exponentiated Normal base measure used for the CMB dataset; this is likely due to the low noise and variance of the CCS dataset.  Computational details are given in Appendix \ref{sec:ccs}.

Results on this dataset were more varied than those for the CMB
dataset.  GPs had the best performance overall; on smaller sets of
training data, the DP-GLM outperformed frequentist CART.  Linear
regression, basic DP regression and Bayesian CART all performed
comparatively poorly.  Treed linear models and treed GPs performed
very well most of the time, but had convergence problems leading to
overall higher levels of predictive error.  Convergence issues were
likely caused by the moderate dimensionality (8 covariates) of the
dataset.  See Figure \ref{fig:ccs} and Table \ref{tab:ccs} for
results.

\begin{figure}
  \begin{center}
    \includegraphics*[width=6in]{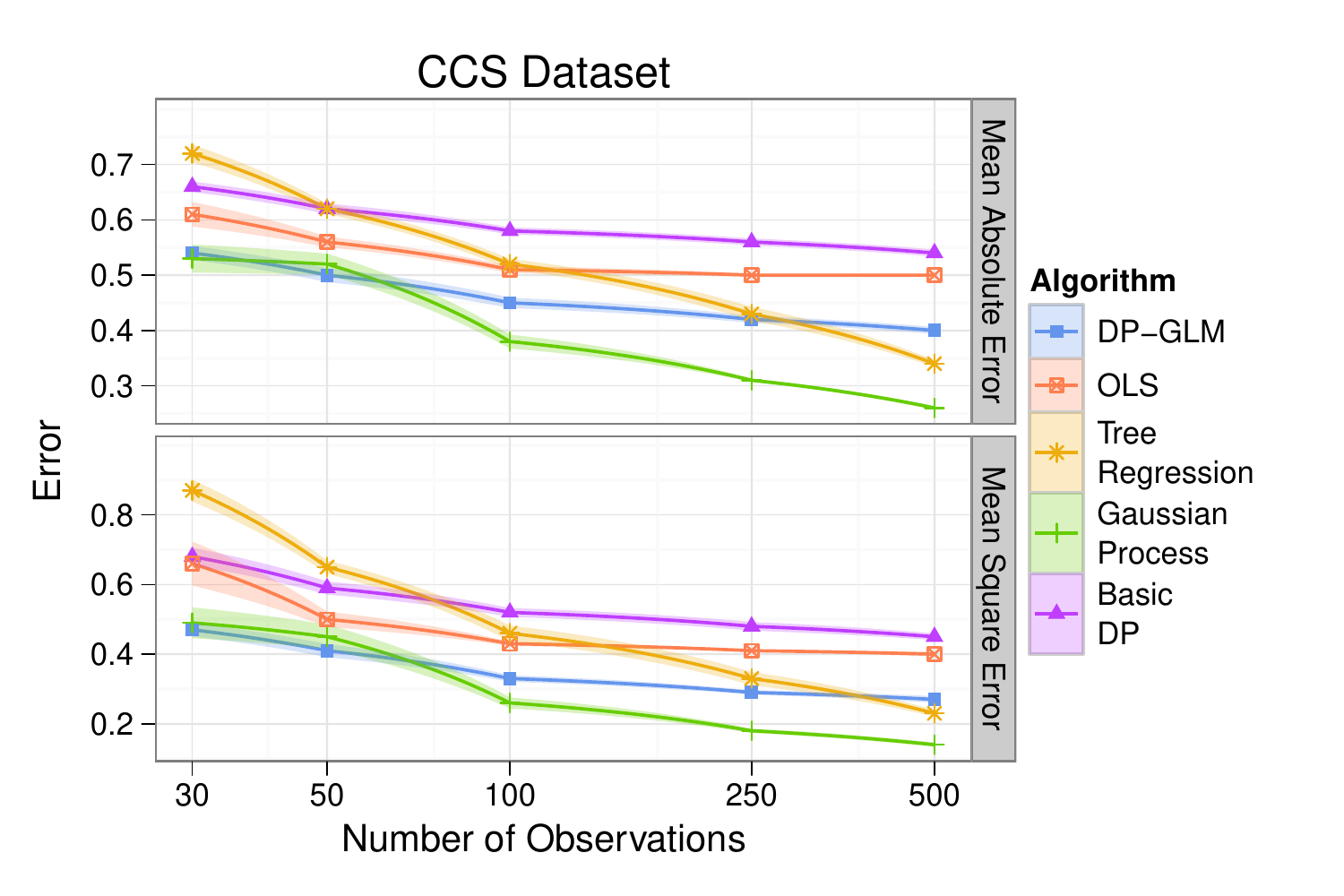}
  \end{center}
    \caption{The average mean absolute error (top) and mean squared
      error (bottom) for ordinary least squares (OLS), tree
      regression, Gaussian processes, location/scale DP and the DP-GLM
      Poisson model on the CCS data set. The data were
      normalized. Mean $+/-$ one standard deviation are given for each
      method.}\label{fig:ccs}
  \end{figure}

\begin{table}
\begin{sc}
\small{
\begin{center}
\begin{tabular}{l | c c c c c | c c c c c }\hline
Method & \multicolumn{5}{c|}{Mean Absolute Error}& \multicolumn{5}{c}{Mean Square Error}  \\
 & 30 & 50 & 100 & 250 & 500& 30 & 50 & 100 & 250 & 500\\\hline
DP-GLM & {\bf0.54}& $0.50 $ & $0.45$ & $0.42$ & $0.40 $
& {\bf0.47}& $0.41 $ & $0.33 $ & $0.28 $ & $0.27 $\\
Location/Scale DP & 0.66 & 0.62 & 0.58 & 0.56 & 0.54 
& 0.68 & 0.59 & 0.52 & 0.48 & 0.45\\
Linear Regression & 0.61 & 0.56 & 0.51 & 0.50 & 0.50 
& 0.66 & 0.50 & 0.43 & 0.41 & 0.40\\
CART & 0.72 & 0.62 & 0.52 & 0.43 & 0.34
& 0.87 & 0.65 & 0.46 & 0.33 & 0.23\\
Bayesian CART &0.78 & 0.72 & 0.63 & 0.55 & 0.54 
& 0.95 & 0.80 & 0.61 & 0.49 & 0.46 \\
Treed Linear Model &1.08 & 0.95 & 0.60 & 0.35 & 1.10
& 7.85 & 9.56 & 4.28 & 0.26 & 1232 \\
Gaussian Process & 0.53 & 0.52 & {\bf0.38} & 0.31 & 0.26 
& 0.49 & 0.45 & {\bf0.26} & 0.18 & 0.14\\
Treed GP & 0.73 & {\bf0.40} & 0.47 & {\bf0.28} & {\bf0.22}
& 1.40 & {\bf0.30} & 3.40 & {\bf0.20} & {\bf0.11} \\
\hline
\end{tabular}
\end{center}
}\normalsize
\end{sc}
\caption{Mean absolute and square errors for methods on the CCS data set by training data size.  The best results for each size of training data are in bold.}
\label{tab:ccs}
\end{table}%

\paragraph{Solar Flare Results.}
The Solar dataset was chosen to demonstrate the flexibility of
DP-GLM. Many regression techniques cannot accommodate categorical
covariates and most cannot accommodate a count-type response.  For
this dataset, we used the following DP-GLM,
\begin{align}\notag
P & \sim DP(\alpha \G_0),\\\notag
\theta_i \g P & \sim P,\\\notag
X_{i,j} \g \theta_i & \sim (p_{i,j,1},\dots,p_{i,j,K(j)}),\\\notag
Y_i \g \theta_i & \sim Poisson\left(\beta_{i,0} + \sum_{j=1}^d \sum_{k = 1}^{K(j)} \beta_{i,j,k} \1_{\{X_{i,j} = k\}}\right). 
\end{align}We used a conjugate covariate base measure and a Gaussian base measure for $\beta$,
\begin{align}\notag
(p_{j,1},\dots,p_{j,K(j)}) & \sim Dirichlet(a_{j,1},\dots,a_{j,K(j)}), & \beta_{j,k} & \sim N(m_{j,k},s_{j,k}^2).
\end{align}Computational details are given in Appendix \ref{sec:solar}.

The only other methods that can handle this dataset are CART,
Bayesian CART and Poisson regression.  The DP-GLM had the best
performance under both error measures (with Bayesian CART a close
second).  The high mean squared error values suggests that frequentist
CART overfit while the high mean absolute error for Poisson regression
suggests that it did not adequately fit nonlinearities.  See Figure
\ref{fig:solar} and Table \ref{tab:solar} for results.

\begin{figure}
\begin{center}
\includegraphics*[width=6in]{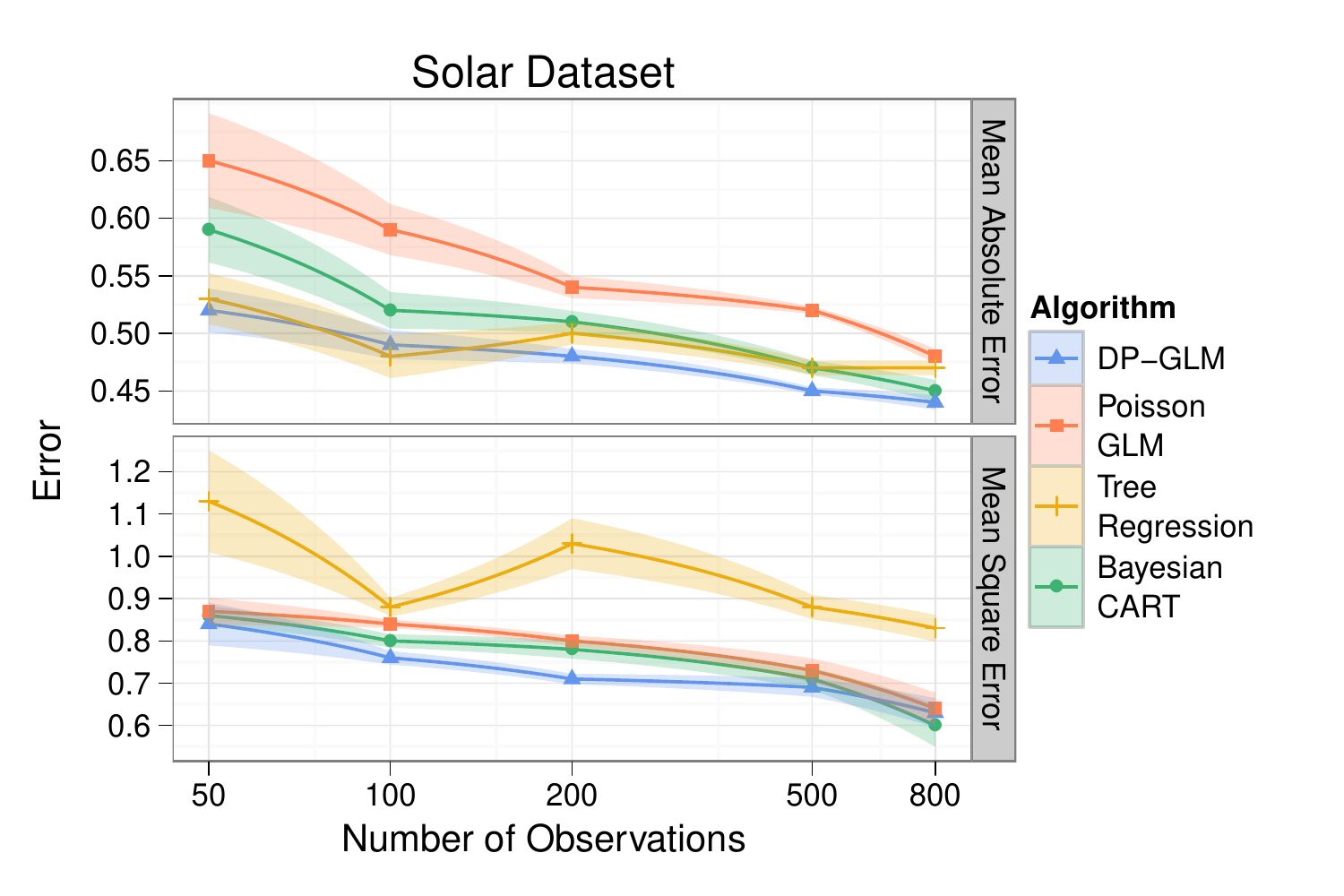}
\end{center}
\caption{The average mean absolute error (top) and mean squared error (bottom) for tree regression, a Poisson GLM (GLM) and DP-GLM on the Solar data set. Mean $+/-$ one standard deviation are given for each method.}\label{fig:solar}
\end{figure}

\begin{table}
\begin{sc}
\small{
\begin{center}
\begin{tabular}{l | c c c c c | c c c c c }\hline
Method & \multicolumn{5}{c|}{Mean Absolute Error}& \multicolumn{5}{c}{Mean Square Error}  \\
 & 50 & 100 & 200 & 500 & 800& 50 & 100 & 200 & 500 & 800\\\hline
DP-GLM & {\bf0.52} & 0.49 & {\bf0.48} & {\bf0.45} & {\bf0.44}
& {\bf0.84} & {\bf0.76} & {\bf0.71} & {\bf0.69} & 0.63\\
Poisson Regression & 0.65 & 0.59 & 0.54 & 0.52 & 0.48
& 0.87 & 0.84 & 0.80 & 0.73 & 0.64\\
CART & 0.53 & {\bf0.48} & 0.50 & 0.47 & 0.47
& 1.13 & 0.88 & 1.03 & 0.88 & 0.83\\
Bayesian CART & 0.59 & 0.52 & 0.51 & 0.47 & 0.45 
& 0.86 & 0.80 & 0.78 & 0.71 & {\bf0.60} \\
\hline
\end{tabular}
\end{center}
}\normalsize
\end{sc}
\caption{Mean absolute and square errors for methods on the Solar data set by training data size.  The best results for each size of training data are in bold.}
\label{tab:solar}
\end{table}

\paragraph{Discussion.}
The DP-GLM is a relatively strong competitor on all of the datasets,
but only CART and Bayesian CART match its flexibility.  It was more
stable than most of its Bayesian competitors (aside from GPs) on the
CCS dataset.  Our results suggest that the DP-GLM would be a good
choice for small sample sizes when there is significant prior
knowledge; in those cases, it acts as an automatic outlier detector
and produces a result that is similar to a Bayesian GLM.  Results from
Section \ref{sec:dpglm} suggest that the DP-GLM is not appropriate
for problems with high dimensional covariates; in those cases, the
covariate posterior swamps the response posterior with poor numerical
results.


\section{Conclusions and Future Work}\label{sec:discussion}

We developed the Dirichlet process mixture of generalized linear
models (DP-GLM), a flexible Bayesian regression technique.  We
discussed its statistical and empirical properties; we gave conditions for
asymptotic unbiasedness and gave situations in which they hold;
finally, we tested the DP-GLM on a variety of datasets against state
of the art Bayesian competitors.  The DP-GLM was competitive in most setting and provided stable, conservative estimates, even with extremely small sample sizes.

One concern with the DP-GLM is computational efficiency as implemented.  All results were generated using MCMC, which does not scale well to large datasets.  An alternative implementation using variational inference~\citep{BlJo05}, possibly online variational inference~\citep{Sa01}, would greatly increase computational feasibility for large datasets.  

Our empirical analysis of the DP-GLM has implications for regression
methods that rely on modeling a joint posterior distribution of the
covariates and the response.  Our experiments suggest that the
covariate posterior can swamp the response posterior, but careful
modeling can mitigate the effects for problems with low to moderate
dimensionality.  A better understanding would allow us to know when
and how such modeling problems can be avoided.


\renewcommand{\theequation}{A-\arabic{equation}}
\setcounter{equation}{0}  
\renewcommand{\thesubsection}{A-\arabic{subsection}}
\setcounter{subsection}{0}
\renewcommand{\thetheorem}{A-\arabic{theorem}}
\setcounter{theorem}{0}
\section*{Appendix}  
\subsection{Posterior Inference}\label{sec:inference}
In the Gibbs sampler, the state is the collection of labels $(z_1,\dots,z_n)$ and parameters $(\theta^*_1,\dots,\theta^*_K)$, where $\theta_c^*$ is the parameter associated with cluster $c$ and $K$ is the number of unique labels given $z_{1:n}$.  In a collapsed Gibbs sampler, all or part of $(\theta^*_1,\dots,\theta^*_K)$ is eliminated through integration.  Let $z_{-i} = (z_1,\dots,z_{i-1},z_{i+1},\dots,z_n)$.  A basic inference algorithm is given in Algorithm \ref{alg:inference}
\begin{algorithm}[h]\caption{Gibbs Sampling Algorithm for the DP-GLM}
  \begin{algorithmic}[1]\label{alg:inference}
  \REQUIRE Starting state $(z_1,\dots,z_n)$, $(\theta^*_1,\dots,\theta^*_K)$, convergence criteria.
  \REPEAT
  \FOR{$i=1$ to $n$}
  \STATE Sample $z_i$ from $p(z_i \g D, z_{-i},\theta^*_{1:K})$.
  \ENDFOR
  \FOR{$c = 1$ to $K$}
  \STATE Sample $\theta_c^*$ given $\{(X_i,Y_i) \, : \,  z_i = c\}$.
  \ENDFOR
  \IF{Convergence criteria are met}
  \STATE Record $(z_1,\dots,z_n)$ and $(\theta^*_1,\dots,\theta^*_K)$.
  \ENDIF
  \UNTIL{$M$ posterior samples obtained.}
  \end{algorithmic}
\end{algorithm}
Convergence criteria for our numerical examples are given in Appendix sections \ref{sec:cmb} to \ref{sec:solar}.  See \citet{GeCaSt04} for a more complete discussion on convergence criteria.

We can sample from the distribution $p(z_i \g D, z_{-i},\theta^*_{1:K})$ as follows,
\begin{align}\label{eq:pOfZ}
p(z_i \g D, z_{-i},\theta_{1:K}^*) & \propto p(z_i \g z_{-i}) p(X_i \g z_{1:n}, D, \theta^*_{1:K}) p(Y_i \g X_i, z_{1:n}, D, \theta^*_{1:K}).
\end{align}The first part of Equation (\ref{eq:pOfZ}) is the Chinese Restaurant Process posterior value,
\begin{equation}\notag
p(z_i\g z_{-i} ) = \left\{
\begin{array}{ll}
 \frac{n_{z_j}}{n-1+\alpha} &  \mathrm{if \ } z_i = z_j \ \mathrm{for \ some \ } j \neq i,  \\
 \frac{\alpha}{n-1+\alpha} & \mathrm{if \ } z_i \neq z_j \ \mathrm{for \ all \ } j \neq i.    
\end{array}
\right.
\end{equation}
Here $n_{z_j}$ is the number of elements with the label $z_j$.  The second term of Equation (\ref{eq:pOfZ}) is the same as in other Gibbs sampling algorithms.  If possible, the component parameters $\theta^*_{1:K}$ can be integrated out (in the case of conjugate base measures and parameters that pertain strictly to the covariates) and $p(X_i \g z_{1:n}, D, \theta^*_{1:K})$ can be replaced with $$\int p(X_i\g z_{1:n}, D,\theta^*_{1:K}) p(\theta^*_{1:K} \g z_{1:n}) d\theta^*_{1:K}.$$  The third term of Equation (\ref{eq:pOfZ}) is not found in traditional Dirichlet process mixture model samplers.  In some cases, this term can also be collapsed, such as Gaussian model with a Normal-Inverse-Gamma base measure.  In that case,\small{
\begin{align}\notag
p(Y_i \g X_i, z_c, D_c) & = \frac{\Gamma((n_n+1)/2)}{\Gamma(n_n/2)} \left(n_n s_n\right)^{-1/2} \exp\left(-1/2 (n_n+1) \log\left(1 + \frac{1}{n_n s_n} (Y_i - m_n)^2\right)\right),\\\notag
\tilde{V} & = \left( V^{-1} + \tilde{X}_c^T \tilde{X}_c\right)^{-1},\\\notag
\hat{m}_n & = \tilde{V}\left(m_0V^{-1}+   \tilde{X}_c^T Y_c\right)\\\notag
m_n & = \tilde{X}_i \hat{m}_n,\\\notag
n_n & = n_{y0} + n_c,\\\notag
s_n^2 & = 4\left(s_{y0}^2 + 1/2 \left( m_0 V^{-1} m_0^T + Y_c^T Y_c - \hat{m}_n^T \tilde{V}^{-1} \hat{m}_n\right)\right)/\left((n_{y0}+n_c) \tilde{X}_c \tilde{V} \tilde{X}_c^T\right)
\end{align}}\normalsize Here, we define $\tilde{X}_c = \{ [1 X_j] \, : \, z_j = z_c\}$, $Y_c = \{ Y_j \, : \, z_j = z_c\}$, $\tilde{X}_i = [1 X_i]$, $n_c$ is the number of data associated with label $z_c$ and the base measure is define as,
\begin{align}\notag
\sigma_y^2 & \sim Inverse-Gamma(n_{y0}, s_{y0}^2),\\\notag
\beta \g \sigma_y^2 & \sim N(m_0, \sigma_y^2 V).
\end{align}

\subsection{Proof of Theorem \ref{thm:generalConvergence}}\label{sec:main}
The outline for the proof of Theorem \ref{thm:generalConvergence} is as follows:
\begin{itemize}
	\item[1)] Show consistency under the set of conditions given in part $(i)$ of Theorem \ref{thm:generalConvergence}.
	\item[2)] Show that consistency and conditions $(ii)$ and $(iii)$ imply the existence and convergence of the expectation.
	\begin{itemize}
		\item[a)] Show that weak consistency implies pointwise convergence of the conditional density.
		\item[b)] Note that pointwise convergence of the conditional density implies convergence of the conditional expectation.
	\end{itemize}
\end{itemize}

The first part of the proof relies on a theorem by \citet{Sc65}.
\begin{theorem}[\citet{Sc65}] Let $\Pi^f$ be a prior on $\mathcal{F}$.  Then, if $\Pi^f$ places positive probability on all neighborhoods
\begin{equation}\notag
\left\{f \, : \, \int f_0(x,y) \log \frac{f_0(x,y)}{f(x,y)} dx dy < \delta\right\}
\end{equation}for every $\delta > 0 $, then $\Pi^f$ is weakly consistent at $f_0$.
\end{theorem}Therefore, consistency of the posterior at $f_0$ follows directly from condition $(i)$.

We now show pointwise convergence of the conditional densities.  Let $f_n(x,y)$ be the Bayes estimate of the density under $\Pi^f$ after $n$ observations,
\begin{equation}\notag
f_n(x,y) = \int_{\mathcal{F}} f(x,y) \Pi^f\left(df\g (X_i,Y_i)_{i=1}^n\right).
\end{equation}
Posterior consistency of $\Pi^f$ at $f_0$ implies that $f_n(x,y)$ converges weakly to $f_0(x,y)$, pointwise in $(x,y)$.
\begin{lemma}\label{lem:conditionalDensity}
Suppose that $\Pi^f$ is weakly consistent at $f_0$.  Then $f_n(y|x)$ converges weakly to $f_0(y|x)$.
\end{lemma}
\begin{proof} 
If $\Pi^f$ is weakly consistent at $f_0$, then $f_n(x,y)$ converges weakly to $f_0(x,y)$.  We can integrate out $y$ to see that $f_n(x)$ also converges weakly to $f_0(x)$.  Then note that
\begin{equation}\notag
\lim_{n\rightarrow \infty} f_n(y|x) = \lim_{n \rightarrow \infty} \frac{f_n(x,y)}{f_n(x)}  = \frac{f_0(x,y)}{f_0(x)}  = f_0(y|x).
\end{equation}

\end{proof}

Now we are ready to prove Theorem \ref{thm:generalConvergence}.
\begin{proof}[Theorem \ref{thm:generalConvergence}]
Existence of expectations: by condition $(iii)$ of Theorem \ref{thm:generalConvergence}, $$\int |y|^{1+\epsilon} f_n(y|x) dy < \infty$$ almost surely.  This and condition $(ii)$ ensure uniform integrability and that $\E_{\Pi^f}[Y\g X = x, (X_i,Y_i)_{i=1}^n]$ exists for every $n$ almost surely.  By Lemma \ref{lem:conditionalDensity}, the conditional density converges weakly to the true conditional density; weak convergence and uniform integrability ensure that the expectations converge: $\E_{\Pi^f}[Y\g X = x, (X_i,Y_i)_{i=1}^n]$ converges pointwise in $x$ to $\E_{f_0}[Y \g X = x]$.

\end{proof}

\subsection{Proof of Theorem \ref{thm:continuous} and Results for Base Measures}\label{sec:gaussian}
We sketch the proof of this theorem in two parts.  First, we show that we can use a Gaussian convolution approximation to $f_0(x,y)$, and second that the posterior of the convolution is weakly consistent at $f_0$.  We do this only for the case where $x$ is a scalar, but it can be easily extended to the vector case.

Uniform equicontinuity will be used throughout this proof and the next.
\begin{definition}A set of functions $\mathcal{F}$ is \textit{uniformly equicontinuous} if for every $\epsilon > 0$, there exists a $\delta >0$ such that for every $x_1,x_2 \in \mathcal{X}$, if $||x_1-x_2|| < \delta$, then $| f(x_1) - f(x_2) | < \epsilon$ for all $f \in \mathcal{F}$.
\end{definition}

\begin{theorem}\label{thm:compact}
Suppose that $f_0(x,y)$ has compact support; that is, there exists some $x^\prime$ and $y^\prime$ such that $f_0(x,y) = 0$ whenever $(x,y) \notin [-x^\prime,x^\prime] \times [-y^\prime, y^\prime]$.  Set
\begin{equation}\notag
f_{0,h}(x,y) = \int \int \int f_0(\mu,\beta_0,\beta_1) \phi_{h_x}(x-\mu) \phi_{h_y}(y - \beta_0 - \beta_1 \mu) d\mu d\beta_0 d\beta_1,
\end{equation}where $\phi_h(x)$ is the Gaussian pdf with variance $h$, evaluated at $x$.  Then $$\lim_{h_x,h_y \rightarrow 0} \int \int f_0(x,y) \log \frac{f_0(x,y)}{f_{0,h}(x,y)} dxdy = 0.$$
\end{theorem}
\begin{proof}
This proof follows Remark 3 of \citet{GhGhRa99} closely.  Fix $0 < c_1 \leq c_2$.  Because $f_0$ is compactly supported, there exists $a^\prime < b^\prime$, $a^{\prime\prime} < b^{\prime \prime}$, and $a^{\prime \prime \prime} < b^{\prime \prime \prime}$ such that for every $$(\mu,\beta_0,\beta_1) \in [a^{\prime}, b^{\prime}] \times [a^{\prime\prime}, b^{\prime\prime}] \times [a^{\prime\prime\prime}, b^{\prime\prime\prime}]= \mathcal{D},$$ $f_0(\mu, \beta_0 + \beta_1\mu) >c_1$ and $$(\mu,\beta_0,\beta_1) \notin \mathcal{D},$$ $f_0(\mu,\beta_0+\beta_1\mu) < c_2.$  We can choose $h_{0,x}$, $h_{0,x}$ such that $N(0,h_{0,x})$ gives 1/3 probability to $(0, b^\prime-a^\prime).$  Define 
\begin{align}\notag
d_1 & = \max_{\beta_0 \in [a^{\prime\prime},b^{\prime\prime}], \beta_1\in [a^{\prime\prime\prime},b^{\prime\prime\prime}]} \left(\beta_0 + \beta_1 a^\prime\right), &  d_2 & = \min_{\beta_0 \in [a^{\prime\prime},b^{\prime\prime}], \beta_1\in [a^{\prime\prime\prime},b^{\prime\prime\prime}] }\left(\beta_0 + \beta_1 b^\prime\right),\\\notag
d_1^\prime & = \min_{\beta_0 \in [a^{\prime\prime},b^{\prime\prime}], \beta_1\in [a^{\prime\prime\prime},b^{\prime\prime\prime}] }\left(\beta_0 + \beta_1 a^\prime\right), &  d_2^\prime & = \max_{\beta_0 \in [a^{\prime\prime},b^{\prime\prime}], \beta_1\in [a^{\prime\prime\prime},b^{\prime\prime\prime}] }\left(\beta_0 + \beta_1 b^\prime\right).
\end{align}By choice of $c_1, c_2$, we can ensure that $d_1 < d_2$.  We can choose $h_{0,y}$ such that $N(0,h_{0,y})$ gives 1/3 probability to $(0,d_2^\prime-d_1^\prime)$.  Then we compute lower bounds on the convolutions for each part of our partitioned space.  Let $h_x < h_{0,x}$ and $h_y < h_{0,y}$.  If $(x,y) \in [a^\prime,b^\prime]\times[d_1^\prime,d_2^\prime],$
\begin{align}\notag
f_{0,h}(x,y) & \geq \int_{a^\prime}^{b^\prime} \int_{a^{\prime\prime}}^{b^{\prime\prime}}\int_{a^{\prime\prime\prime}}^{b^{\prime\prime\prime}} f_0(\mu,\beta_0+\beta_1\mu) \phi_{h_x}(x-\mu) \phi_{h_y}(y - \beta_0 - \beta_1\mu) d\mu d\beta_0d\beta_1 \\\notag
& \geq c_1(\Phi((b^\prime-x)/h_x)+\Phi((x-a^\prime)/h_x))(\Phi((y-d_1)/h_y)+\Phi((d_2-y)/h_y))\\\notag
&  \geq c_1/9.
\end{align} Now consider $(x,y) \in [a^\prime,b^\prime]\times[d_2^\prime,\infty),$
\begin{align}\notag
f_{0,h}(x,y) & \geq \int_{a^\prime}^{b^\prime} \int_{a^{\prime\prime}}^y\int_{a^{\prime\prime\prime}}^y f_0(\mu,\beta_0+\beta_1\mu) \phi_{h_x}(x-\mu) \phi_{h_y}(y - \beta_0 - \beta_1\mu) d\mu d\beta_0d\beta_1 \\\notag
& \geq f_0(x,y)(\Phi((b^\prime-x)/h_x)+\Phi((x-a^\prime)/h_x))(1/2+ \Phi((d_2^\prime-d_1^\prime)/h_y)-1)\\\notag
&  \geq f_0(x,y)/9.
\end{align}The negative tail and both cases where $x \notin [a^\prime,b^\prime], y \in [d_1^\prime,d_2^\prime]$ can be bounded in the same manner.  Now consider $(x,y) \in [b^\prime,\infty) \times[ d_2^\prime,\infty),$
\begin{align}\notag
f_{0,h}(x,y) & \geq \int_{a^\prime}^{x} \int_{a^{\prime\prime}}^y\int_{a^{\prime\prime\prime}}^y f_0(\mu,\beta_0+\beta_1\mu) \phi_{h_x}(x-\mu) \phi_{h_y}(y - \beta_0 - \beta_1\mu) d\mu d\beta_0d\beta_1 \\\notag
& \geq f_0(x,y)(1/2+ \Phi((b^\prime-a^\prime)/h_x)-1)(1/2+ \Phi((d_2^\prime-d_1^\prime)/h_y)-1)\\\notag
&  \geq f_0(x,y)/9.
\end{align}We can do this for the other three regions to generate the function,
\begin{equation}\notag
g(x,y) = 
\left\{
\begin{array}{l l}
 \log(3f_0(x,y)/c_1) & \mathrm{if}\ x,y \in  [a^\prime,b^\prime]\times[d_1^\prime,d_2^\prime],  \\
  \log(9) & \mathrm{otherwise}. 
\end{array}
\right.
\end{equation}The function $g(x,y)$ dominates $\log(f_0/f_{0,h})$ and is $\p_{f_0}$ integrable.  Using dominated convergence and a variant of Fatou's lemma, $\int f_0 \log(f_0/f_{0,h}) \rightarrow 0 $ as $h_{x}, h_y \rightarrow 0$.
\end{proof}


\begin{theorem}\label{thm:continuousCon}
Define $f_{h,P}(x,y)$ as
$$f_{h,P}(x,y) = \int \phi_{h_x}(x-\mu)\phi_{h_y}(y-\beta_0-\beta_1x) dP(\mu,\beta_0,\beta_1).$$  Suppose that $f_0(x,y)$ has compact support and all values of $(\mu,\beta_0,\beta_1,h_x,h_y)$ are in the weak support of $\Pi$.  Then for all $\epsilon > 0$,
\begin{equation}\notag
\Pi\left\{ P : \int f_0(x,y) \log \frac{f_0(x,y)}{f_{h,P}(x,y)} dxdy < \epsilon\right\} > 0.
\end{equation}
\end{theorem}
\begin{proof}This proof is similar to the one presented in \citet{GhGhRa99}.  Consider the case where $x$ is scalar; this assumed for simplicity and can easily be extended to the vector-valued case.  Fix $\epsilon > 0$.  
Assume $f_{0,h}(x,y)$ is as above.  Note that
\begin{equation}\label{eq:bigMess}
\int f_0(x,y) \log\frac{f_0(x,y)}{f_{h,P}(x,y)} = \int f_0(x,y) \log \frac{f_0(x,y)}{f_{0,h}(x,y)} + \int f_0(x,y) \log \frac{f_{0,h}(x,y)}{f_{h,P}(x,y)}.
\end{equation}
Using Theorem \ref{thm:compact}, we can pick $h_{0,x}, h_{0,y}$ such that for every $h_x < h_{0,x}$, $h_{y} < h_{0,y}$,
\begin{equation}
\int f_0(x,y)\log \frac{f_0(x,y)}{f_{0,h}(x,y)} < \epsilon/2.
\end{equation}

Now we focus on the second term of Equation (\ref{eq:bigMess}).  Set $$\mathcal{D} = [-x^\prime, x^\prime] \times [-y^\prime, y^\prime] \times [-y^\prime, y^\prime].$$  Note that
\begin{align}\notag
\int & f_0(x,y) \log \frac{f_{0,h}(x,y)}{f_{h,P}(x,y)}\\\notag
& \leq \int_{-x^\prime}^{x^\prime} \int_{-y^\prime}^{y^\prime} f_0(x,y) \log \frac{\int_{\mathcal{D}} f_0(\mu,\beta_0+\beta_1\mu)\phi_{h_x}(x-\mu)\phi_{h_y}(y-\beta_0-\beta_1x) d\mu d\beta_0 d\beta_1}{\int_{\mathcal{D}} \phi_{h_x}(x-\mu)\phi_{h_y}(y-\beta_0-\beta_1x)dP(\mu,\beta_0,\beta_1)}dxdy.
\end{align}

Define
\begin{equation}\notag
\mathcal{F} = \left\{\phi_{h_x}(x-\mu) \phi_{h_y}(y - \beta_0 - \beta_1 x)\, : \, x \in [-x^\prime, x^\prime], \, y \in [-y^\prime, y^\prime]\right\}.
\end{equation}$\mathcal{F}$ is a set of functions in $(\mu,\beta_0,\beta_1) \in \mathcal{D}$; it is uniformly equicontinuous.  Given this fact, we can use the Arzela-Ascoli theorem to find a finite tiling of $(x,y)$ to approximate $g_0(x,y)$ arbitrarily well in the region $x \in [-x^\prime,x^\prime],$ $y \in [-y^\prime, y^\prime]$.  The rest of the proof proceeds as in \citet{GhGhRa99}.
\end{proof}Theorem \ref{thm:continuousCon} can easily be extended to include scale mixtures of $\sigma_x, \sigma_y$~\citep{GhGhRa99,To06}.  Theorem \ref{thm:continuous} follows as a direct result of Theorem \ref{thm:continuousCon}.

All of the base measures proposed in Subsection \ref{subsec:example} have weak support over the desired space, $\R^d \times \R^{d+1} \times \R_+^{d+1}$, and hence satisfy Theorem \ref{thm:continuous}.


\subsection{Proof of Theorem \ref{thm:multinomial} and Results for Base Measures}\label{sec:multi}
We follow the same plan for sketching the proof of Theorem \ref{thm:multinomial}.
\begin{theorem}\label{thm:multiBounded}
Suppose that $f_0(x,k)$ has compact support.  Define 
\begin{equation}\notag
f_{0,h}(x,k) = \int f_0(\mu,k) \phi_h(x-\mu) du.
\end{equation}Then,
\begin{equation}\notag
\lim_{h\rightarrow 0} \int \int f_0(x,k) \log \frac{f_0(x,k)}{f_{0,h}(x,k)} dxdh = 0.
\end{equation}
\end{theorem} The proof is similar to that of Theorem \ref{thm:compact}.

\begin{theorem}\label{thm:multiCon}Define $f_{h,P}(x,k)$ as
$$f_{h,P}(x,k) = \int \phi_{h}(x-\mu)\frac{\exp(\beta_{0,k}-\beta_{1,k}x)}{\sum_{j=1}^K \exp(\beta_{0,j} + \beta_{1,j}x)} dP(\mu,\beta_{0,1},\dots,\beta_{1,K}).$$
Suppose that $f_0(x,k)$ has compact support and all values of $(\mu,\beta_{0,1},\dots,\beta_{1,K},h)$ are in the weak support of $\Pi$.  Then for all $\epsilon > 0$,
\begin{equation}\notag
\Pi\left\{ P : \int f_0(x,k) \log \frac{f_0(x,k)}{f_{h,P}(x,y)} dxdk < \epsilon\right\} > 0.
\end{equation}
\end{theorem}
\begin{proof}Assume $f_{0,h}(x,k)$ is as above.  Note that
\begin{equation}\label{eq:bigMess2}
\int f_0(x,k) \log\frac{f_0(x,k)}{f(x,k)} = \int f_0(x,k) \log \frac{f_0(x,k)}{f_{0,h}(x,k)} + \int f_0(x,k) \log \frac{f_{0,h}(x,k)}{f_{h,P}(x,k)}.
\end{equation}
Using Theorem \ref{thm:compact}, we can pick $h_{0}$ such that for every $h < h_{0,x}$,
\begin{equation}
\int f_0(x,k)\log \frac{f_0(x,k)}{f_{0,h}(x,k)} < \epsilon/4.
\end{equation}Now we focus on the second term of Equation (\ref{eq:bigMess2}).  We can make $f_{0,h}(x,k)$ arbitrarily close to a density with bounded response probabilities that are still continuous in $x$.  Choose $\gamma > 0$; for the sake of simplicity, assume $k = 1,2$ (the outcome space can easily be expanded).  Writing $f_0(x,k) = f_0(x)f_0(k|x)$, set
$$f_\gamma(k|x) = \left\{
\begin{array}{ll}
 f_0(k|x) & \mathrm{if \ } f_0(k|x) \in [\gamma, 1-\gamma],   \\
  \gamma & \mathrm{if \ } f_0(k|x) \in [0,\gamma), \\
  1 - \gamma & \mathrm{if \ } f_0(k|x) \in (1-\gamma, 1].   
\end{array}
\right.$$Then, for any $x \in [-x^\prime,x^\prime]$ and $k \in \{1,2\}$,
\begin{equation}\notag
\left| \frac{\int f_0(\mu)f_{0}(k| \mu) \phi_h(x-\mu)d\mu}{\int f_0(\mu)f_{\gamma}(k| \mu) \phi_h(x-\mu)d\mu}-1 \right| < 2 \gamma.
\end{equation}Note that
\begin{align}\notag\int f_0(x,k) \log \frac{f_{0,h}(x,k)}{f_{h,P}(x,k)}&  = \int f_0(x,k) \log \frac{f_{0,h}(x,k)}{f_{h,\gamma}(x,k)}+ \int f_0(x,k) \log \frac{f_{h,\gamma}(x,k)}{f_{h,P}(x,k)}, \\\notag
& \leq \int f_0(x,k) \log \frac{f_{h,\gamma}(x,k)}{f_{h,P}(x,k)} + 2 \gamma.
\end{align}Finally, we approximate $f_{h,\gamma}$ by something that has the desired GLM form.  Fix $0 < \delta < \gamma/(4K)$.  Define
\begin{equation}\notag
g_{\delta}(y_1,\dots,y_k|x) =  \max\left(0, \delta\sum_{k=1}^K f_{\gamma}(k|x) - \sum_{k=1}^K \left|\frac{\exp(y_k)}{\sum_{j=1}^K \exp(y_j)} - f_{\gamma}(k|x)\right|\right),
\end{equation}for $y \in [-b,b]$ for some bound $b$.  This produces a triangular box of ``close'' conditional probabilities around the true conditional probability.  Define $f_{h,\gamma,\delta}$ as
\begin{align}\notag
f_{h,\gamma,\delta}(x,k) & = \int_{-x^\prime}^{x^\prime} \int_{[-b_0,b_0]^{K}} \int_{[-b_1,b_1]^K} f_0(\mu) \frac{g_{\delta}(\beta_{0,1}+\beta_{1,1}\mu,\dots,\beta_{0,K}+\beta_{1,K}\mu|\mu)}{\int_{[-b_0,b_0]^K} \int_{[-b_1,b_1]^K} g_{\delta}(\beta_{0}+\beta_{1}\mu,k|\mu)d\beta_0d\beta_1}\\\label{eq:ugly}
& \times \phi_h(x-\mu) \frac{\exp(\beta_{0,k}+\beta_{1,k}\mu)}{\sum_{j=1}^K \exp(\beta_{0,j}+\beta_{1,j}\mu)}d\mu d\beta_{0,1}\dots d\beta_{1,K}.
\end{align}Because $f_\gamma(k|x)$ is bounded away from $0$ and $1$, we can find finite $b_0$, $b_1$ such that $$\max_{\beta_{0}\in[-b_0,b_0]^K, \beta_1 \in [-b_1,b_1]^K} g_\delta(\beta_{0,1}+\beta_{1,1}x, \dots, \beta_{0,K}+\beta_{1,K}x|x) = \delta $$ for every $x \in [-x^\prime,x^\prime]$.  Let $L$ be the maximum Lipschitz constant of $f_0(k|x)$ for $x\in [-x^\prime,x^\prime]$; these exist because $f_0(k|x)$ is continuous in $x$, which has a compact domain.  Then, observe that
\begin{equation}\notag
\left|\frac{\int f_0(\mu)f_{\gamma}(k|\mu) \phi_{h}(x-\mu) d\mu}{f_{h,\gamma,\delta}} - 1 \right| \leq 2 K L \delta.
\end{equation}Therefore,
\begin{align}\notag
\int f_0(x,k) \log \frac{f_{h,\gamma}(x,k)}{f_{h,P}(x,k) }&  = \int f_0(x,k) \log \frac{f_{h,\gamma}(x,k)}{f_{h,\gamma,\delta}(x,k)} + \int f_0(x,k) \log \frac{f_{h,\gamma,\delta}(x,k)}{f_{h,P}(x,k)},\\\notag
& \int f_0(x,k)\log \frac{f_{h,\gamma,\delta}(x,k)}{f_{h,P}(x,k)} + \leq 2KL \delta.
\end{align}

Set
\begin{equation}\notag
\mathcal{F} = \left\{ \phi_h(x-\mu) \frac{\exp(\beta_{0,k} + \beta_{1,k}x)}{\sum_{j=1}^K \exp(\beta_{0,j} + \beta_{i,j}x)} : x\in [-x^\prime, x^\prime], k \in \{1,\dots,K\}\right\}.
\end{equation}It is equicontinuous for $(\mu,\beta_0,\beta_1) \in [-x^\prime,x^\prime]\times [-b_0,b_0]^K \times [-b_1,b_1]^K$; using this fact, we proceed in the same manner as in Theorem \ref{thm:continuousCon}.

\end{proof}
Theorem \ref{thm:multinomial} follows from Theorem \ref{thm:multiCon}.  As in the Gaussian example, all of the base measures given have sufficient weak support.

\subsection{CMB Computational Details}\label{sec:cmb}
The DP-GLM was run on the largest data size tested several times; log posterior probabilities were evaluated graphically, and in each case the posterior probabilities seem to have stabilized well before 1,000 iterations.  Therefore, all runs for each sample size were given a 1,000 iteration burn-in with samples taken every 5 iterations until 2,000 iterations had been observed.  The scaling parameter $\alpha$ was given a Gamma prior with an initial value set at 1.  The means and variances of each component and all GLM parameters were also given a log-normal hyper distribution.  The model was most sensitive to the hyper-distribution on $\sigma_y$, the GLM variance.  Small values were used ($\log(m_y) \sim N(-3,2)$) to place greater emphasis on response fit.  The non-conjugate parameters were updated using the Hamiltonian dynamics method of \citet{Ne10}.

Models with both conjugate and log-normal base measures were tried on this dataset.  The conjugate base measure model did not capture heteroscedasticity well, so only the log-normal base measure was used for the remaining tests on this dataset.

This model was implemented in Matlab; a run on the largest dataset took about 500 seconds.

\subsection{CCS Computational Details} \label{sec:ccs}
Again, the DP-GLM was run on the largest data size tested several times; log posterior probabilities were evaluated graphically, and in each case the posterior probabilities seem to have stabilized well before 1,000 iterations.  Therefore, all runs for each sample size were given a 1,000 iteration burn-in with samples taken every 5 iterations until 2,000 iterations had been observed.  The scaling parameter $\alpha$ was given a Gamma prior with an initial value set at 1.  The hyperparameters of the conjugate base measure were set manually by trying different settings over four orders of magnitude for each parameter.

All base measures were conjugate, so the sampler was fully collapsed.  $\alpha$ was updated using Hamiltonian dynamics~\citep{Ne10}.  Original results were generated by Matlab; the longest run times were about 1000 seconds.  This method has been re-implemented in Java in a highly efficient manner; the longest run times are now under about 10 seconds.  Run times would likely be even faster if variational methods were used for posterior sampling~\citep{BlJo05}.

\subsection{Solar Computational Details}\label{sec:solar}
Again, the DP-GLM was run on the largest dataset size tested several times; log posterior probabilities were evaluated graphically, and in each case the posterior probabilities seem to have stabilized well before 1,000 iterations.  Therefore, all runs for each sample size were given a 1,000 iteration burn-in with samples taken every 5 iterations until 2,000 iterations had been observed.   The scaling parameter $\alpha$ was set to 1 and the Dirichlet priors to $Dir(1,1,\dots,1)$.  The response parameters were given a Gaussian base distribution with a mean set to 0 and a variance chosen after trying parameters with four orders of magnitude.

All covariate base measures were conjugate and the $\beta$ base measure was Gaussian, so the sampler was collapsed along the covariate dimensions and used in the auxiliary component setting of Algorithm 8 of \citet{Ne00}.  The $\beta$ parameters were updated using Metropolis-Hastings.  Results were in generated by Matlab; run times were substantially faster than the other methods implemented in Matlab (under 200 seconds).

\bibliographystyle{agsm}

\end{document}